%% file: arxiv.tex
\documentclass[nohyperref]{article}
\def\withnotes{0}
\def\doublecolumn{0}
\usepackage{fullpage}
\usepackage{microtype}
\usepackage{graphicx}
\usepackage{subfigure}
\usepackage{booktabs} 

\usepackage{hyperref}

\usepackage{natbib}
\usepackage{algorithm}
\usepackage{algorithmic}

\usepackage{amsmath}
\usepackage{amssymb}
\usepackage{mathtools}
\usepackage{amsthm}

\usepackage[capitalize,noabbrev]{cleveref}

\theoremstyle{plain}
\newtheorem{theorem}{Theorem}[section]

\newtheorem{lemma}[theorem]{Lemma}
\newtheorem{corollary}[theorem]{Corollary}
\theoremstyle{definition}
\newtheorem{definition}[theorem]{Definition}

\theoremstyle{remark}

\allowdisplaybreaks

\usepackage[textsize=tiny]{todonotes}

\ifnum\withnotes=1
    \newcommand{\thnote}[1]{{\color{red} theertha: #1}}
    \newcommand{\ylnote}[1]{{\color{blue} Yuhan: #1}}
\else
    \newcommand{\thnote}[1]{}
    \newcommand{\ylnote}[1]{}
\fi
\newcommand{\ignore}[1]{}

\newcommand{\clip}{\text{clip}}
\newcommand{\EE}{\mathbb{E}}
\newcommand{\eps}{\varepsilon}
\newcommand{\Bin}{\text{Bin}}
\newcommand{\Poi}{\text{Poi}}
\newcommand{\Lap}{\text{Lap}}

\newcommand{\indic}[1]{{\mathbf{1}_{#1}}}
\newcommand{\Var}{\text{Var}}
\newcommand{\C}{C}
\newcommand{\lambdaAvg}{\bar{\lambda}}

\newcommand{\randClip}{\text{rand-clip}}
\newcommand{\bh}{\bar{h}}
\newcommand{\bN}{\HistSum}
\newcommand{\Cmax}{{C_m}}
\newcommand{\bp}{\mathbf{p}}
\newcommand{\Dir}{\text{Dir}}

\newcommand{\HistSum}{\overline{N}}

\title{Algorithms for bounding contribution for histogram estimation under user-level privacy}

\author{%
  \begin{tabular}[t]
{c@{\extracolsep{6.5em}}c}
  Yuhan Liu\footnote{Part of the work was done during an internship at Google.} & Ananda Theertha Suresh\\
 Cornell University & Google Research\\ 
\small \texttt{yl2976@cornell.edu} & \small \texttt{theertha@google.com}
\end{tabular}
\vspace{2ex}\\
\begin{tabular}[t]{c@{\extracolsep{6.5em}}c@{\extracolsep{6.5em}}c}
  Wennan Zhu & Peter Kairouz & Marco Gruteser\\
  Google Research & Google Research & Google Research \\
 \small \texttt{wennanzhu@google.com} & \small\texttt{kairouz@google.com}
 & \small \texttt{gruteser@google.com}
\end{tabular}
}

\begin{document}
\maketitle
\input{content}

\end{document}

%% file: content.tex
\begin{abstract}
We study the problem of histogram estimation under user-level differential privacy, where the goal is to preserve the privacy of \textit{all} entries of any single user. 
We consider the heterogeneous scenario where the quantity of data can be different for each user. In this scenario, the amount of noise injected into the histogram to obtain differential privacy is proportional to the maximum user contribution, which can be amplified by few outliers. One approach to circumvent this would be to bound (or limit) the contribution of each user to the histogram. However, if users are limited to small contributions, a significant amount of data will be discarded. In this work, we propose algorithms to choose the best user contribution bound for histogram estimation under both bounded and unbounded domain settings. When the size of the domain is bounded, we propose a user contribution bounding strategy that almost achieves a two-approximation with respect to the best contribution bound in hindsight. For unbounded domain histogram estimation, we propose an algorithm that is logarithmic-approximation with respect to the best contribution bound in hindsight. This result holds without any distribution assumptions on the data. Experiments on both real and synthetic datasets verify our theoretical findings and demonstrate the effectiveness of our algorithms. We also show that clipping bias introduced by bounding user contribution may be reduced under mild distribution assumptions, which can be of independent interest.
\end{abstract}

\section{Introduction}

Differential privacy (DP)~\citep{CynthiaFKA06} provides a rigorous formulation of privacy and has been applied to many algorithmic and learning tasks that involve the access to private and sensitive information. Notable applications include private data release~\citep{hardt2010simple}, learning histograms~\citep{CynthiaFKA06}, statistical estimation~\citep{DiakonikolasHS15, KamathLSU18, acharya2020differentially, KamathSU20, AcharyaCT19, AcharyaSZ18a}, and machine learning~\citep{chaudhuri2011differentially, bassily2014private, mcmahan2017learning, DworkTTZ14, abadi2016deep}. 

In several applications, each user may contribute many data samples to a dataset. For example, one may have multiple health records in a hospital, or may type many words on their phone's virtual keyboard. Naturally, users would hope that  \textit{all} of their private data is protected. To achieve this, private algorithms should guarantee \textit{user-level} differential privacy.

However, many existing works assume that each user only contributes one data sample. Thus, an algorithm designed under this assumption can only be used to protect the privacy of each data sample but not the user. In other words, such algorithms achieve \textit{item-level} privacy, but they cannot protect privacy at the user level and may not meet the increasing privacy concerns in most applications where users may contribute a lot of data. Therefore, there has been a growing interest in revisiting differential privacy in the \textit{user-level} setting.

User-level privacy is much more stringent than the item-level counterpart. Under user-level privacy, the amount of noise added is dependent on the \emph{sensitivity} of the differential privacy mechanism, which is typically the maximum number of items contributed by any single user. Hence, even if a majority of users contribute little data and very few outliers contribute a large amount of data, then the amount of noise added will be significantly large. This begs an important question: should we limit the user contribution while providing differential privacy? This question was initiated by \citet{amin2019bounding} in the context of empirical risk minimization. They showed that not restricting user contribution results in a large amount of noise injection and restricting user contribution to achieve low sensitivity, may result in loss of a large amount of useful data and suffer from bias.   
With this observation, they provided an algorithm that determines near-optimal user contribution bound. In this work, we study the problem of bounding user contribution for general differentially private histogram estimation. 

Histogram estimation is a fundamental problem that arises in many real-world applications such as demographic data and user preferences. For example, \citet{chen2019federated} used histogram estimation to compute unigram language models via \textit{federated learning} \citep{mcmahan17fedavg, kairouz2019advances}. Beyond machine learning, \textit{federated analytics} uses histogram estimation to support the Now Playing feature on Google’s Pixel phones, a tool that shows users what song is playing in the room around them \cite{ramage2020fablog}. 

Histogram estimation can be broadly divided into two categories: estimation over bounded domains and unbounded domains. In some examples such as estimating unigram language models over finite known vocabulary, the size of the domain is finite and can be counted. We refer to this scenario as bounded domain histogram estimation. In some examples such as finding all possible words used in English, the size of the domain is the set of all strings and hence unbounded or extremely large. We refer to this scenario as unbounded histogram estimation.

There is abundant literature on histogram estimation in the context of item-level privacy, such as \citep{hay2009boosting, Suresh2019differentially, xu2013differentially}. Little has been known, however, under user-level privacy. Recently, \citep{Liu2020learning, levy2021learning}, studied the problem of histogram estimation under user-level privacy when data from users are generated from near-identical distributions.  Since users' data may come from diverse distributions in practice, we cannot leverage techniques from these works. 
Motivated by the need for algorithms that work well with heterogenous user data, we ask the following question:
\vspace{-0.75ex}
\begin{center}
\emph{Can we design private algorithms to find a (nearly) optimal bound of user contributions for histogram estimation?}
\end{center}
\vspace{-0.75ex}
Somewhat surprisingly, despite many recent works in the area, the above question has not been extensively studied. We take a step towards answering the above question by designing algorithms that perform well in the heterogeneous setting 
where both the number of samples and data distribution can be unknown and different across users (hence both need to be viewed as private information). 

Our contributions are as follows. We first study the problem of bounded domain histogram estimation, where the domain size is finite and can be enumerated efficiently. In this setting, we propose private user contribution bounding algorithms that obtain a factor two approximation compared to the best contribution bound in hindsight. We then study the problem of unbounded domain histogram estimation, where the domain size is very large and cannot be enumerated efficiently. In this setting, we propose a private user contribution bounding algorithm that achieves a logarithmic approximation compared to the best algorithm in hindsight\footnote{We emphasize that our near-optimality results are relative to a specific family of DP algorithms, i.e. those that follow the clipping and additive noise recipe. We do not claim optimality over all possible DP mechanisms.}. Finally, we investigate if the bias introduced by these user contribution bounding algorithms can be reduced by post-processing techniques and show that under mild non i.i.d. distribution assumptions, the amount of bias can be reduced by simple post processing techniques. We also provide a complete proof of the gap 
between the \textit{debiased} and the \textit{non-debiased} algorithms. We evaluate these algorithms on standard federated datasets and demonstrate the practicality of the algorithms.




The paper is organized as follows. In Section~\ref{sec:related}, we discuss related work. In Section~\ref{sec:problem_formulation}, we introduce the definition and problem formulation. In Section~\ref{sec:estimation_no_assumption} we introduce our algorithms for the bounded domain setting with no i.i.d. assumptions. In Section~\ref{sec:open-domain} we describe our algorithms for unbounded domain histogram estimation. In Section~\ref{sec:bias_reduction} we briefly introduce our debiasing algorithm and its guarantees. In Section~\ref{sec:experiments} we show the experiment results. In Section~\ref{sec:conclusion}, we conclude with a discussion about how to extend our methods to federated settings.

\section{Related work}
\label{sec:related}
Given its importance, user-level privacy has been studied by several works in the last decade. 
One of the primary motivations for user-level privacy is federated learning, where the goal is to learn a model at the server while keeping the raw data on edge devices such as cell phones \citep{mcmahan17fedavg, kairouz2019advances}.  Ensuring privacy at the user level is a crucial concern in federated learning. Even though users do not send their original data, various works~\citep{phong2017privacy,wang2019beyond} have shown it is still possible to reconstruct user's data if additional privacy-preserving mechanisms are not used. Therefore, user-level privacy has been studied under various machine learning tasks in the federated learning setup~\citep{mcmahan2017learning, mcmahan2018general, augenstein2019generative}.  
Indeed, understanding the fundamental privacy-utility trade-offs under user-level privacy is one of the main challenges in federated learning~\citep[Section 4.3.2]{kairouz2019advances}. 

Several works studied fundamental theoretical problems in user-level private learning. \citet{ghazi2021user} studied PAC learnability under user-level privacy. \citet{levy2021learning} studied high-dimensional distribution estimation and optimization and designed efficient algorithms. Both works require i.i.d. data and assume a fixed number of samples across users. There are several recent works closely related to user-level private histogram estimation. \citet{amin2019bounding} studied the inherent bias and variance trade-off in bounding user contributions under user-level privacy for empirical risk minimization. Their analysis applies to estimating the total count of one symbol in the aggregate histogram. We extend their work to the setting of $d > 1 $ symbols.

\citet{Liu2020learning} and \citet{levy2021learning} studied a closely related problem of discrete distribution estimation and designed optimal algorithms in terms of user complexity (the minimum number of users required to learn an unknown distribution with given accuracy) up to logarithmic factors. However, their analysis assumes that all users' data are drawn from nearly identical distributions. Furthermore, the algorithms algorithms in \citet{Liu2020learning} may be impractical due to time inefficiency and large constants in user complexity. 

\citet{esfandiari2021tight} studied robust high dimensional mean estimation under user-level privacy, assuming i.i.d. and fixed number of samples across users. While their algorithm is robust to at most 49\% of the samples being arbitrarily adversarial, their result still requires that the remaining samples are independent and identically distributed.  
\citet{cummings2021mean} studied mean estimation of Bernoulli random variables, allowing different distributions and number of samples for different users. Their setup can be viewed as a special case of histogram estimation when the domain size is 2. However, no theoretical guarantee is provided when the domain size is larger than 2. \citet{wilson2019differentially} proposed differentially private SQL with bounded user contributions.

\citet{huang2021instance} provides an instance optimal algorithm for bounded-domain histogram estimation. However, it is unclear how their algorithm extends to the case of unbounded domains. In addition, we prove optimality against the best contribution bound in hindsight \textit{for any fixed dataset}, which is orthognonal to their formulation of neighborhood instance optimality. We discuss the differences and contributions compared to their work in more detail in subsequent sections.


\section{Problem formulation}
\label{sec:problem_formulation}

Differential privacy (DP) is studied in the central and local settings~\citep{CynthiaFKA06, kasiviswanathan2011can, duchi2013local}. In this paper, we study the problem under the lens of central differential privacy, where the goal is ensure the algorithm's outcomes do not reveal too much information about any user's data.  We now define differential privacy, starting with the basic definition of neighboring datasets. Here, we assume the number of users is known and fixed and hence we use the replacement notion of neighboring datasets.

\begin{definition}
Let $D=\{X_1, \ldots, X_n\}$ respresent a dataset of $n$ users. Each $X_i$ consists of $m_i$ samples $\{X_{i, j}\}_{j=1}^{m_i}$. Let $D'=\{X_i'\}_{i=1}^n$ be another dataset. We say $D$ and $D'$ are neighboring (or adjacent) datasets if for some $j\in[n]$, 
\[
X_i=X'_{i}, \text{for all } i\ne j.
\]
\end{definition}

\begin{definition}[Differential privacy]
A randomized mechanism $\mathcal{M}$ with  range $\mathcal{R}$ satisfies $(\eps, \delta)$-differential privacy if for any two adjacent datasets 
 $D, D'$ and for any subset of output 
 $\mathcal{S} \subseteq \mathcal{R}$, it holds that 
 \[
 Pr[\mathcal{M}(D) \in \mathcal{S}] \leq e^{\eps} Pr[\mathcal{M}(D') \in \mathcal{S}] + \delta.
 \]
 \end{definition}
 If $\delta=0$, then the privacy is also referred to as \textit{pure DP}, and for simplicity we say that the algorithm satisfies $\eps$-DP. If $\delta>0$, we refer to it as \textit{approximate DP}.

We consider the following problem. There are $n$ users and user $i$ has a histogram $N_i=(N_{i, 1}, \ldots, N_{i, d})\in\mathbb{Z}_{\geq 0}^d$ over a discrete domain of size $d\in\mathbb{N}$. Without loss of generality, we can assume the domain to be $[d]:=\{1, \ldots, d\}$. Let $m_i=\|N_i\|_1$ be the size of histogram $N_i$. The dataset $D$ is the collection of users' histograms. The goal is to estimate the population-level histogram, i.e., the sum of the histograms
$$\HistSum(D)=\sum_{i=1}^n N_i.$$

 

We make no assumptions about the distribution and size of each user's histogram.
Given an $(\varepsilon, \delta)$-differentially private algorithm 
whose output histogram is $\widehat{N}$, we characterize its performance by the expected $\ell_1$ distance between the algorithm output and the true population-level histogram
\[
\EE \| \HistSum - \widehat{N}\|_1 = \sum^d_{j=1} \EE |\HistSum_j - \widehat{N}_j|,
\]
where the expectation is over the randomization in the differential privacy algorithm. 


\section{Optimal user contribution for histograms over bounded domains}
\label{sec:estimation_no_assumption}
We first consider the problem of estimating the population-level histogram when the size of the domain $d$ is small enough to be enumerated. Let $\|\cdot\|_q$ denote the $\ell_q$ norm. For a vector $x\in\mathbb{R}^d$ and $C\in\mathbb{R}^+$, define the $\ell_q$ clipping function, $$
\clip_q(x, C)=\frac{C\cdot x}{\max(C, \|x\|_q)}.$$
A standard strategy is to \emph{clip} each user contribution either in $\ell_1$ (when $\delta=0$) or $\ell_2$ norm (when $\delta>0$) and add a suitable amount of Laplace or Gaussian noise respectively \cite{CynthiaFKA06,pmlr-v80-balle18a}. For completeness, the details are shown in Algorithm~\ref{alg-clipping}. 
In the rest of the paper, we use the term clipping and bounding user contribution interchangeably. 
When $\delta>0$, choosing an appropriate noise level $\sigma=\sigma(\eps, \delta)$ guarantees $(\eps, \delta)$-differential privacy.

\begin{algorithm}
\caption{Bounded domain histogram estimation}
\label{alg-clipping}
\begin{algorithmic}[1]
\STATE Input: histograms $N_1, \ldots, N_n$, clip threshold $C$, privacy parameter $\eps, \delta$, noise level $\sigma=\sigma(\eps, \delta)$.
\STATE Clipping: for each user $i$, do
\[
\widetilde{N}_i^{1} = \clip_1(N_i, C) \quad \text{ and } \quad \widetilde{N}_i^{2} = \clip_2(N_i, C) 
\]
\STATE If $\delta = 0$, return 
$
\widehat{N}_{L} = \sum^n_{i=1} \widetilde{N}^1_i + \text{Lap}(C/\eps)$.
\STATE If $\delta > 0$, return 
$
\widehat{N}_{G} = \sum^n_{i=1} \widetilde{N}^2_i + \text{N}(0, C^2\sigma^2\mathbb{I})$.
\end{algorithmic}
\end{algorithm}

\begin{lemma}
    \label{lem:dp-noise}[\cite{CynthiaFKA06,pmlr-v80-balle18a}]
    When $\delta=0$, Algorithm~\ref{alg-clipping} guarantees $\eps$-DP. When $\delta>0$. When $\eps\le1$, choosing $\sigma^2=2\log(1.32/\delta)/\eps^2$ guarantees $(\eps, \delta)$-DP for Algorithm~\ref{alg-clipping}. When $\eps>1$, choosing $\sigma=\alpha/\sqrt{2\eps}$ where $\alpha$ is defined in \citet[Algorithm 1]{pmlr-v80-balle18a} guarantees $(\eps, \delta)$-DP.
\end{lemma}

For $\delta>0$, it is noted that the expression for $\eps>1$ is much more complicated that $\eps\le1$, thus for simplicity we mainly focus on $\eps\le 1$.

\subsection{Selecting the optimal threshold non-privately}

There is a bias-variance trade-off in choosing the clipping threshold $C$. If $C$ is small, then the noise magnitude (or variance) is small, but the clipped histogram would have large error (or bias). On the other hand, if $C$ is large, the clipped histogram would be more accurate (less bias), but the added noise would be large (high variance). 
For any dataset $D$, we provide an accurate characterization of the best threshold that balances the bias and variance for both the Laplace and Gaussian estimator. For the Laplace estimator, the proof is similar to that of \citet{amin2019bounding} and is omitted.
\begin{lemma}
\label{lem:laplac-2opt}
Let $\mathcal{L}_L(C, D)=\EE[\|\widehat{N}_
L-\HistSum(D)\|_1]$. For any dataset $D$, choosing $C^*$ as the top $\lceil d/\eps \rceil$ element in $\{m_i\}_{i=1}^n$ yields 2-approximation
\[
\mathcal{L}_L(C^*, D)\le 2\inf_{C\ge 0}\mathcal{L}_L(C, D),
\]
where the expectation is over the Laplace mechanism.
\end{lemma}
We now state the result for the Gaussian mechanism in Theorem~\ref{thm:gaussian-2opt}. The complete proof is in Appendix~\ref{sec:gaussian-2opt-proof}. 
\begin{theorem}
\label{thm:gaussian-2opt}
Let $\mathcal{L}_G(C, D)=\EE[\|\widehat{N}_
G-\HistSum(D)\|_1]$. Let $\eps \leq 1$ and  $M=d\sigma\sqrt{\frac{2}{\pi}}$. For any dataset $D$, choosing $C^*$ such that
\[
C^*=\arg\min_{C\ge 0}\left\{\sum_{i:\|N_i\|_2>C}\frac{\|N_i\|_1}{\|N_i\|_2}\le M\right\}
\]
yields 2-approximation, 
\[
\mathcal{L}_G(C^*, D)\le 2\inf_{C\ge 0}\mathcal{L}_G(C, D).
\]
\end{theorem}

\subsection{Choosing the optimal threshold privately}
We now discuss how to find $C^*$ privately using an additional privacy budget of $(\eps', \delta')$,
and further provide complete guarantees in terms of excess error compared to $2\inf_{C\ge 0}\mathcal{L}_G(C, D)$. 
We emphasize that it only requires a very small extra privacy budget compared to the original $(\eps, \delta)$ to achieve good performance, as we will later show in the experiments. For $\delta=0$, one can levarage many existing algorithms to privately find the $d/\eps$ quantile in Lemma~\ref{lem:laplac-2opt} \citep[Theorem 2]{dick2023subset}. Thus we mainly focus on the $\delta>0$ case. 

Note that computing the optimal $C$ in a differentially private way is possible though difficult 
because the sensitivity of $\sum_{i: \|N_i\|_2 > C} \frac{\|N_i\|_1}{\|N_i\|_2}$ can be very large for some datasets. 
However, observe that by Cauchy-Schwarz inequality,
\[
\|N_i\|_1/\|N_i\|_2 \leq \sqrt{\|N_i\|_0}.
\]
Hence, if each user's histogram has very few non-zero entries, then the sensitivity would be low. 

We observe this to be the case in practice. 
To illustrate, we plot the unique number of symbols contributed by each user in Sentiment140 \citep{sentiment140} and SNAP datasets in Figure~\ref{fig:sparse}. 
Observe that even though $d$ is large, most users have fewer than $200$ samples. 
Hence, we assume that each user's histogram is at most $s$ sparse.
Under this assumption, the sensitivity is upper bounded by $\sqrt{s}$. Note that we can simply set $s=d$ if the bound is not known. 

\begin{figure}
    \ifnum\doublecolumn=1
        \centerline{\includegraphics[width=0.7\linewidth]{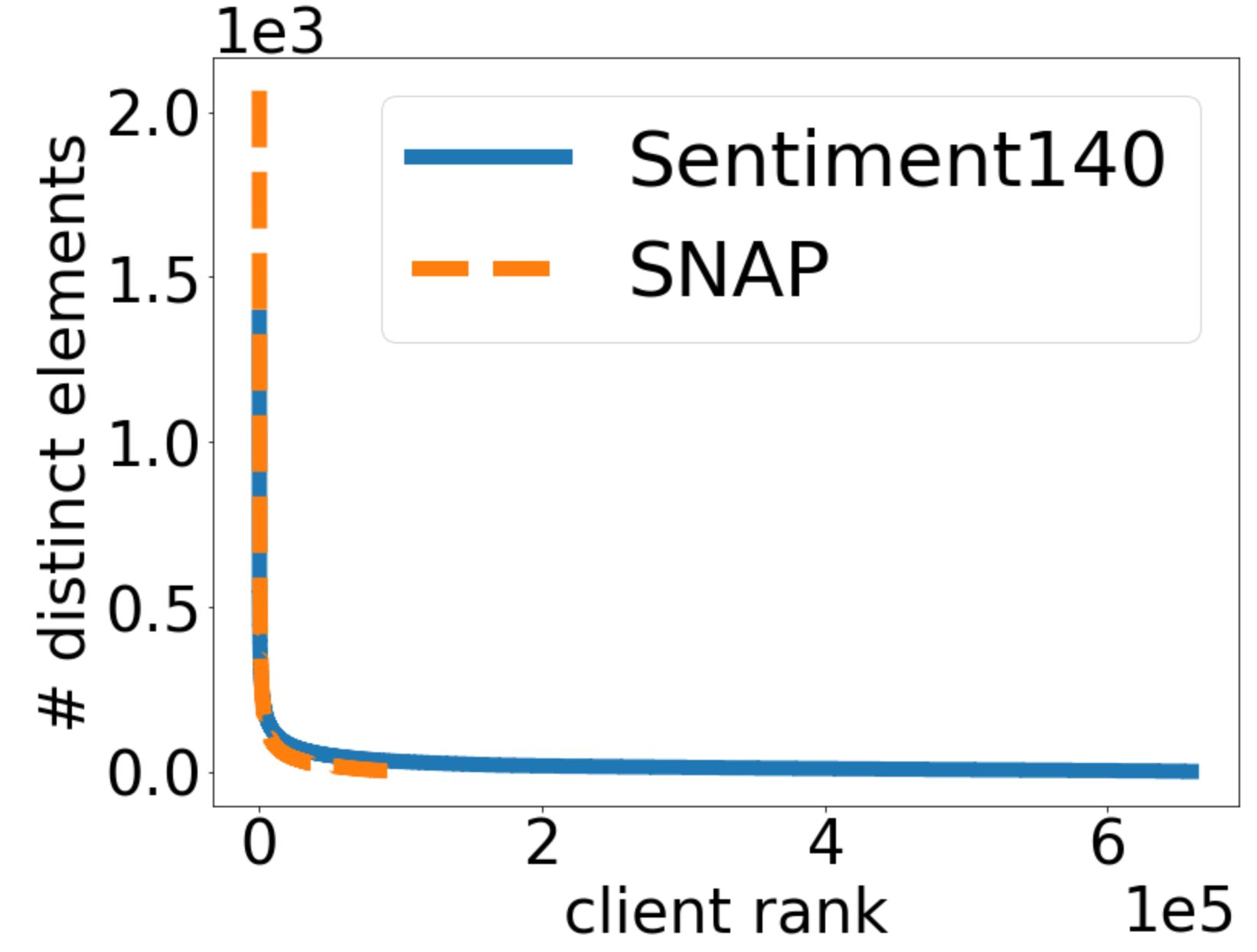}}
    \else
        \centerline{\includegraphics[width=0.4\linewidth]{sparse.png}}
    \fi
  \caption{Number of distinct words in each client. The x-axis is the rank of clients ordered from the highest count of distinct words to the lowest.}
  \label{fig:sparse}
\end{figure}

We first note that $C^*$ can be written as a minimizer to a convex function, 
\[
G(C) = \sum_{i=1}^nf_i(C)+C M,
\]
where $f_i(C) =\max\left\{1-\frac{C}{\|N_i\|_2}, 0\right\}\|N_i\|_1$.  Hence we can use techniques from differentially private convex optimization algorithms. We consider two such algorithms and provide their corresponding guarantees. 

\paragraph{Estimating $C^*$ with DP-SGD.}

We first consider the DP-SGD algorithm \citep[Algorithm 1]{bassily2014private} to estimate $C$ by minimizing $G(C)$. Using \citet[Theorem 2.4]{bassily2014private}, we have the following guarantee.
\begin{corollary}
\label{cor:bassily}
Let $\Cmax$ be an upper bound on $C^*$ and let $C_{\text{DPSGD}}$ be the output of  \citet[Algorithm 1]{bassily2014private}. Assume that $n\ge d$. Then,  $\EE[\mathcal{L}_G(C_{\text{DPSGD}}, D)]-2\inf_{C\ge 0}\mathcal{L}_G(C, D)
$ is upper bounded by 
\begin{align*}
O\left(\frac{\Cmax\Big(\sqrt{s}+\frac{\sqrt{\log(1/\delta)}}{\eps}\Big)}{\eps'}\log^{3/2}(n/\delta')\sqrt{\log(1/\delta')}\right).
\end{align*}
\end{corollary}
\begin{proof}
The proof directly follows by noting that $G(C)=\sum_{i=1}^n(f_i(C)+\frac{MC}{n}) $. 
Hence the Lipschitz constant for $f_i(C)+MC/n$ w.r.t. $C$ is 
$L=\sqrt{s}+M/n=\sqrt{s}+O\left(\frac{\sqrt{\log(1/\delta)}}{\eps}\right)$.
Since $\Cmax\ge C^*$, setting the domain diameter $\|\mathcal{C}\|_2=\Cmax$ 
in \citet[Theorem 2.4]{bassily2014private} completes the proof.
\end{proof}

\paragraph{Estimating $C^*$ with output perturbation.}
We consider the second algorithm based on output perturbation \citep{chaudhuri2011differentially}, 
which ensures $(\eps', 0)$-DP and is good for small $n$ and $\eps'$. 
Here, we solve a regularized convex optimization problem and perturb the output to provide differential privacy. 
The algorithm is outlined in Algorithm~\ref{alg-convex}.

\begin{algorithm}[h]
\caption{Clipping threshold estimation with output perturbation}
\begin{algorithmic}[1]
\STATE Input: histograms $N_1, \ldots, N_n$, an upper bound of $C^*$ denoted by $\Cmax$, sparsity parameter $s$, privacy parameter $\eps'$.
\STATE Let $\lambda = \frac{2\sqrt{2s}}{\Cmax\sqrt{n\eps}}$ and $\Delta = \frac{4\sqrt{s}}{\lambda n}$ Compute $C'$, the minimizer of $F(C)=\frac{1}{n}G(C)+\frac{\lambda}{2}C^2.$
\STATE Return $C_{\text{output}} = C' + \text{Lap}(\Delta/\eps').$
\end{algorithmic}
\label{alg-convex}
\end{algorithm}

With appropriate parameters, the combined algorithm almost achieves a 2-approximation with respect to the best clipping threshold.
\begin{corollary}
\label{cor:hist}
Algorithm~\ref{alg-convex} is $(\eps', 0)$ differentially private. If $\Cmax$ is an upper bound on $C^*$,  setting $\lambda = \frac{2\sqrt{2s}}{\Cmax\sqrt{n\eps'}}$ yields an error
\[
\EE[\mathcal{L}_G(C_{\text{output}}, D)] \leq 2\inf_{C\ge 0}\mathcal{L}_G(C, D)
+ 
2 \sqrt{2}\,\Cmax \sqrt{\frac{ns}{\eps'}}.
\]
\end{corollary}

Comparing DP-SGD and Algorithm~\ref{alg-convex}, we can see that DP-SGD has a better asymptotic dependence on $n$, 
and Algorithm~\ref{alg-convex} has a better dependence on $\eps'$. Furthermore, DP-SGD provides an approximate DP guarantee and Algorithm~\ref{alg-convex} gives a pure DP guarantee. Finally, the time complexity of DP-SGD is typically $O(n^2)$, however it has been improved recently to $O(n)$ with similar guarantees \cite{feldman2020private}. 

\paragraph{Remark on the sparsity $s$.}  We emphasize that Lemma~\ref{lem:laplac-2opt} and Theorem~\ref{thm:gaussian-2opt} are general results that \textit{do not} require bounded 
 $\ell_0$ norms of histograms. Moreover, the convergence results in Corollary~\ref{cor:bassily} and \ref{cor:hist} follow directly by replacing 
 with $s$ with $d$ if the $\ell_0$ bound is not satisfied. Hence as long as 
 $n$ is sufficiently larger than $d$, the excess error is still small. Empirically, we show that even setting $s=d$ yields a performance close to the true 2-approximation threshold, and is much better than choosing the threshold according to \cite{amin2019bounding}. See Appendix~\ref{sec:s=d-experiment} for details.

\paragraph{Comparison with \citet{huang2021instance}} 
\cite{huang2021instance} proposed a clipping-based algorithm to very similar to ours to minimize $\ell_2$ estimation error. Their algorithm first applies $\ell_2$ clipping and then adds suitable amount of Gaussian noise. They further proved instance optimality over a neighborhood of $D$. However there are several key differences. 


\citet{huang2021instance} proved instance optimality against all differentially private algorithms over some \textit{neighborhood} of $D$, while our results indicate the optimality of all clipping algorithms for any \textit{fixed} dataset. Therefore, the results in \citet{huang2021instance} and our work are orthogonal to each other (in that neither result implies the other) and take different perspectives on the same problem. 

The technical difference is that they focused on minimizing the $\ell_2$ norm instead of the $\ell_1$ norm. As a result, their threshold is a quantile of the $\ell_2$ norms of all users' histograms, which is very different from the optimal threshold in Theorem~\ref{thm:gaussian-2opt}. Not surprisingly, as we will demonstrate in Section~\ref{sec:experiments}, their algorithm does not perform well when the error metric is $\ell_1$. 

\section{Optimal contribution for histograms over unbounded domains}
\label{sec:open-domain}
In the unbounded domain setting, the domain size can be prohibitively large or even infinite, so it is not practical to add noise to all items in the domain. We describe an algorithm for unbounded domain histogram estimation in Algorithm~\ref{alg:open-domain-algorithm} based on the sparse vector technique 
 \cite{dwork2009complexity}. Even though $d$ is very large, the run time of the algorithm depends only on the number of items with non-zero counts. However, in this approach, the privacy guarantee not only depends on the $\ell_1$ norm of the user contribution but also on the $\ell_\infty$ and $\ell_0$ norms. 
 
 While the standard $\ell_1$ clipping defined in the previous section reduces the $\ell_1$ norm, it does not reduce the $\ell_0$ norm of the histogram. Hence, we use the following randomized clipping strategy: for a histogram $N$ and integer $C>0$, let $\randClip(N, C)$ be the histogram obtained by sampling $\min\{\|N\|_1, C\}$ items without replacement. Since all histograms are integer-valued, 
 the $\ell_\infty$ norm of the clipped histogram is upper bounded by the $\ell_1$ norm.

 In this technique, each user first uses $\randClip$ to clip their histogram to ensure $\ell_1$ and $\ell_\infty$ norm to be less than $C$. Then an appropriate amount of Laplace noise is added to each non-zero count. Finally, we delete all items with counts less than a threshold $t$ and output the histogram of the remaining symbols and their noisy counts. 
 The privacy guarantee is stated in Lemma~\ref{thm:open-domain-alg-privacy}.
\begin{algorithm}[t]
    \caption{Unbounded domain histogram estimation}
    \label{alg:open-domain-algorithm}
    \begin{algorithmic}[1]
        \STATE Input: privacy parameters $\eps, \delta$, histograms $N_1, \ldots, N_n$, threshold $C$.
        \STATE $\HistSum=\sum_{i=1}^n N_i$.
        \STATE $t=C+\frac{C}{\eps}\log\frac{C}{2\delta}$.
        \STATE For each user $i$, $h_i=\randClip(N_i, C)$.
        \STATE $\widetilde{N} = \sum_{i=1}^nh_i+Z$ where $Z=[Z_j\mathbf{1}_{\HistSum_j>0}]_{j=1}^{d}$ and $Z_j\sim \Lap(C/\eps)$.
        \STATE Return $\widehat{N}$ where for each item $j\in[d]$ such that $\HistSum_j>0$,
        \[
        \widehat{N}_j=\widetilde{N}_{j}\indic{\widetilde{N}_j>t}.
        \]
    \end{algorithmic}
\end{algorithm}

\begin{lemma}
    Algorithm~\ref{alg:open-domain-algorithm} is $(\eps, \delta)$-diffentially private. 
    \label{thm:open-domain-alg-privacy}
\end{lemma}

\begin{proof}
Note that by random clipping, each $\|\bh_i\|_r\le C$ for $r=0, 1, 2, \infty$. By the guarantee of the Laplace mechanism, $\widetilde{N}$ is $(\eps, 0)$-DP. 
 Recall that the CDF of $Z$ is given by $\Phi(x)=1-\frac{1}{2}e^{-\eps x/C}$ for $x>0$. Thus,
\[
\Phi\left(\frac{C}{\eps}\log\frac{C}{2\delta}\right)=1-\frac{\delta}{C}\ge (1-\delta)^{1/C}.
\]
The final inequality is due to Bernoulli's inequality $(1+x)^r\le 1+rx$ for $x\ge-1$ and $r\in[0, 1]$. Instantiating \citet[Theorem 1]{delta2020github} proves the differential privacy guarantees. 
\end{proof}

Assume that $\Cmax$ is an upper bound on $\|N_i\|_1, i\in[n]$ which could potentially be very large. Then, we only need to focus on $C\le \Cmax$ (we are minimizing the error with respect to $C$, and it is common to assume some bound on optimization variables). In the next theorem we provide a tight characterization of the expected $\ell_1$ error of Algorithm~\ref{alg:open-domain-algorithm} up to logarithmic factors.
\begin{theorem}
\label{thm:open-domain-error-main} Assume that $1\le C \le \Cmax \le e^{\eps/(3\delta)}$, $\delta\le 1/8$, where $\Cmax$ is the maximum contribution of any user before clipping.
\ifnum\doublecolumn=1
\begin{align}
& \frac{1}{2} \sum_{i=1}^n\max\left\{ \|N_i\|_1 - C, 0\right\}\nonumber\\
& +  \frac{1}{ 12\log \frac{\Cmax}{2\delta}} \sum_{j:\bN_j>0} \EE[\min(\bh_j, t)]\leq  \mathbb{E}[\|\widehat{N} - \HistSum\|_1] \leq \nonumber
 \\
& 2 \sum_{i=1}^n\max\left\{ \|N_i\|_1 - C, 0\right\} +
 \sum_{j:\bN_j>0} \EE[\min(\bar{h}_j, t)].
 \label{equ:open-domain-upper-bound}
\end{align}
\else
\begin{align}
& \frac{1}{2} \sum_{i=1}^n\max\left\{ \|N_i\|_1 - C, 0\right\}+  \frac{1}{ 12\log \frac{\Cmax}{2\delta}} \sum_{j:\bN_j>0} \EE[\min(\bh_j, t)] \nonumber
 \\
& \leq  \mathbb{E}[\|\widehat{N} - \HistSum\|_1] \leq2 \sum_{i=1}^n\max\left\{ \|N_i\|_1 - C, 0\right\} +
 \sum_{j:\bN_j>0} \EE[\min(\bar{h}_j, t)].
 \label{equ:open-domain-upper-bound}
\end{align}
\fi
\end{theorem}

Details of the proof is in Appendix~\ref{sec:open-domain-proof}.
We argue that the assumption on $C, \eps$, and $\delta$ is very mild.
$\delta$ is set as $O(1/n)$ and $\eps$ is chosen to be a constant near 1 (say 0.5 to 5),
which implies that the upper bound on $C$ is exponential in $n$.

If $C^*$ minimizes the upper bound in Theorem~\ref{thm:open-domain-error-main}, 
then $C^*$ yields a logarithmic approximation. 
However, the upper bound in~\eqref{equ:open-domain-upper-bound} depends on $C$ directly via the $\sum_{i=1}^n\max\left\{ \|N_i\|_1 - C, 0\right\}$ 
and indirectly via randomly clipped histogram $\bar{h}_j$ and the threshold $t$. 
Furthermore, the expression is non-convex in $C$, thus convex optimization approaches may not yield provable guarantees. 
Hence, to privately estimate $C^*$, one can obtain the function values for all integers $0<C\le \Cmax$
and apply the exponential mechanism with an additional small privacy budget $\eps'$. 
\begin{corollary}
    Let $\mathcal{L}_O(C, D)$ be the expected $\ell_1$ error of Algorithm~\ref{alg:open-domain-algorithm} on dataset $D$ given threshold $C$. Let $\hat{C}$ be the output of the exponential mechanism with additional privacy budget $\eps'$. Then, $\mathcal{L}_O(\hat{C}, D)$ is upper bounded by
    \[
12\log\frac{\Cmax}{2\delta}\inf_{C\ge 1}\mathcal{L}_O(C, D)+\frac{6\Cmax\log \Cmax}{\eps'}.
    \]
\end{corollary}
\begin{proof}
    Let $C_{opt}$ be the threshold that minimizes the $\ell_1$ error. Write the lower and upper bounds in Theorem~\ref{thm:open-domain-error-main} as $L(C)$ and $U(C)$ respectively.
    Using the fact that $C^*$ minimizes \eqref{equ:open-domain-upper-bound} and applying Theorem~\ref{thm:open-domain-error-main}, 
    \begin{align*}
        U(C^*) &\le 12\log\frac{\Cmax}{2\delta}L(C^*)\le 12\log\frac{\Cmax}{2\delta}L(C_{opt})\\
        & \le 12\log\frac{\Cmax}{2\delta}\mathcal{L}_O(C_{opt}, D).
    \end{align*}
    Note that changing the data of one user changes~\eqref{equ:open-domain-upper-bound} by at most $3\Cmax$. By the utility of the exponential mechanism,
    \begin{align*}
        \mathcal{L}_O(\hat{C}, D)
        &\le U(\hat{C}) \le U(C^*) +\frac{6\Cmax\log \Cmax}{\eps'}\\
 &\le 12\log\frac{\Cmax}{2\delta}\mathcal{L}_O(C_{opt}, D)+\frac{6\Cmax\log \Cmax}{\eps'}.
\end{align*}
\end{proof}

In practice, the expectation in \eqref{equ:open-domain-upper-bound} can be hard to compute.
By Jensen's inequality, $$\EE[\min(\bh_j, t)]\le \min(\EE[\bh_j], t).$$ 
Hence, we propose to replace the former with $\min(\EE[\bh_j], t)$. Observe that 
\begin{align*}
\EE[\bh_j] 
&= \EE\left[\sum^n_{i=1} h_{i,j} \right]   = \sum^n_{i=1}\EE[h_{i,j} ]  = \sum^n_{i=1} \frac{CN_{ij}}{\max(C, \|N\|_1)}.
\end{align*}
Hence, 
\if\doublecolumn=1
\begin{align*}
       & V(C) = 2 \sum_{i=1}^n\max\left\{ \|N_i\|_1 - C, 0\right\} \\
       & + \sum_{j:\bN_j>0} \min\left(\sum_i \frac{CN_{ij}}{\max(C, \|N\|_1)}, C+ \frac{C}{\epsilon} \log \frac{C}{2\delta} \right).   
\end{align*}
\else
\begin{align*}
       V(C) = 2 \sum_{i=1}^n\max\left\{ \|N_i\|_1 - C, 0\right\} + \sum_{j:\bN_j>0} \min\left(\sum_i \frac{CN_{ij}}{\max(C, \|N\|_1)}, C+ \frac{C}{\epsilon} \log \frac{C}{2\delta} \right).   
\end{align*}
\fi

Searching over all possible (integer) values of $C$ can also be inefficient.
We can instead search over a subset $\mathcal{C}\subseteq[\Cmax]$. 
We describe the procedure to privately find the best threshold $C$ in Algorithm~\ref{alg:opt-threshold-open-domain}.
\begin{algorithm}
    \caption{Private threshold selection for histograms over unbounded domain}
    \label{alg:opt-threshold-open-domain}
    \begin{algorithmic}[1]
        \STATE Input: privacy parameters $\eps, \delta$ for Algorithm~\ref{alg:open-domain-algorithm}, privacy parameters $\eps', \delta'$ for estimating the optimal threshold, user histograms $N_1, \ldots, N_n$, $\Cmax$.
        \STATE Select a subset $\mathcal{C}$ of $\{1, \ldots, \Cmax\}$. 
        \STATE For each $C\in \mathcal{C}$, compute $V(C)$
        \STATE Return $\hat{C}$, the output of the exponential algorithm with privacy parameter $\eps'$ and sensitivity $5\Cmax/2$ over $\{V(C): C\in\mathcal{C}\}$. 
    \end{algorithmic}
\end{algorithm}

In Section~\ref{sec:experiments}, we empirically demonstrate that
Algorithm~\ref{alg:opt-threshold-open-domain} can also achieve an error very close to the true optimal threshold.
 
\section{Bias reduction}
\label{sec:bias_reduction}

We  prove that the bias from clipping can be significantly reduced 
when $N_i$'s satisfy some mild distribution assumptions
and show that the debiasing method provides improvements 
even on real datasets where these assumptions may not necessarily hold.
Consider the special case of $d=1$, which we refer to as \textit{count estimation}. 
Let $\mathcal{D}$ be a family of distributions over $\mathbb{Z}_{\geq 0}$.
For each user $i$, $N_i$ is drawn independently from some distribution in $\mathcal{D}$ with mean $\lambda_i> 0$. 
$\lambda_i$'s can be arbitrary and do not need to be equal.

In addition to the absolute error of counts $|\widehat{N}-\HistSum|$, we also want to characterize 
the accuracy for estimating the mean $\lambdaAvg=\frac{1}{n}\sum_{i=1}^n\lambda_i$. 
Let $\hat{\lambda}=\widehat{N}/n$ be an estimate of $\lambdaAvg$. We are interested in the expected square error
\[
\EE[(\lambdaAvg-\hat{\lambda})^2],
\]
where the expectation is over the randomness of the algorithm and the dataset.

In this work we set $\mathcal{D}$ to be the family of Poisson distributions since they arise in many applications. 
For example, they can be used to model the occurrences of a memoryless event in a fixed time window. 
Also, they are good approximations of the binomial distribution $\Bin(m, p)$
when $mp$ is a constant \citep{le1960approximation}, 
and can be very useful when estimating the count of one element 
in a histogram over a very large domain (e.g. the count of a particular word).

It is easy to see that clipping inevitably induces bias. In many practical situations,
it is often reasonable to make mild distribution assumptions on user data.
In this section, we ask if the clipping bias can be reduced with such assumptions. 
We answer affirmatively for bounded domain with $d=1$ under non-i.i.d. Poisson assumptions on each user's count. 

Our algorithm is shown in Algorithm~\ref{alg:poisson_single}. 
It essentially adds a post-processing procedure on the output of Algorithm~\ref{alg-clipping} to reduce the clipping bias. 
Since this is a post-processing step, it does not affect the privacy guarantees. 
We show a detailed analysis of the performance of Algorithm~\ref{alg:poisson_single} 
and discuss two possible extensions to high dimensions in the appendix. 

\begin{algorithm}
\caption{Debiasing algorithm for Poisson distribution}
\label{alg:poisson_single}
\begin{algorithmic}[1]
\STATE Input: $N_1, \ldots, N_n$, $\C\in \mathbb{N}$.
\STATE $h(\lambda)=\EE_{X\sim \Poi(\lambda)}[\clip(X, \C)]$
\STATE $Y_i=\clip(N_i, \C)$
\STATE Return $\widehat{N}=g\left(\sum_{i=1}^nY_i+Lap(\C/\eps)\right)$, where $g(y)=nh^{-1}(y/n)$
\end{algorithmic}
\end{algorithm}

\section{Experiments}
\label{sec:experiments}
\begin{figure*}[t]
\centering
    \includegraphics[width=0.35\textwidth]{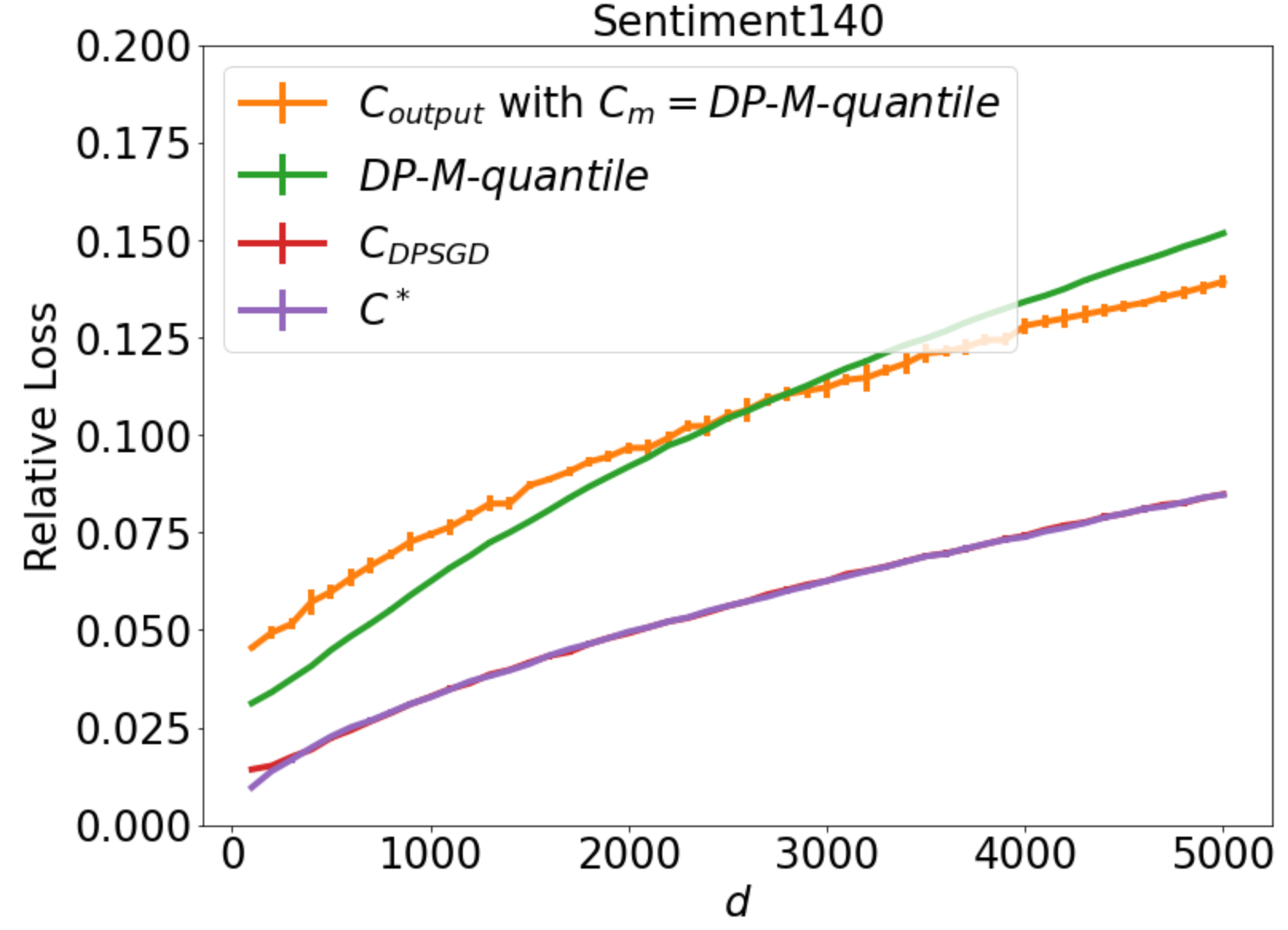}$\quad$
    \includegraphics[width=0.35\textwidth]{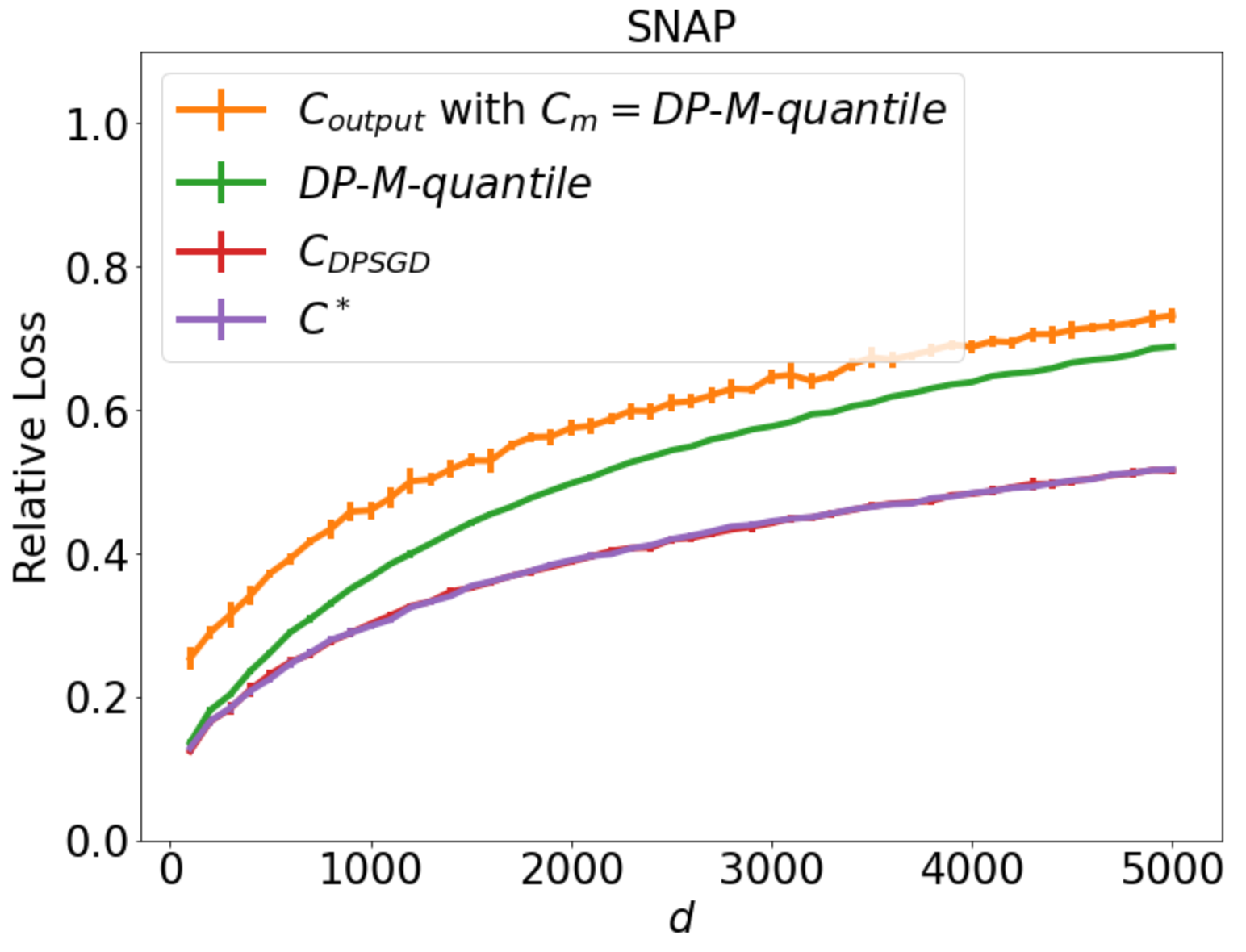}
  \caption{Histogram estimation over bounded domains. \textbf{Left:} Sentiment140 \textbf{Right:} SNAP.}
  \label{fig:hist-experiment}
\end{figure*}

\begin{figure*}[t]
    \centering
    \includegraphics[width=0.34\linewidth]{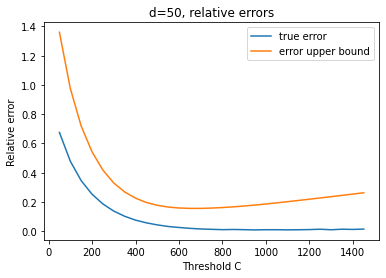} 
    \includegraphics[width=0.34\linewidth]{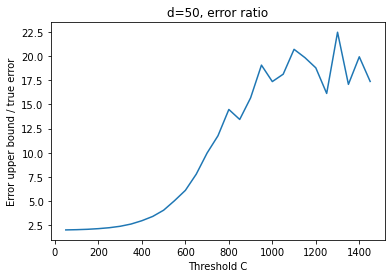}
    \includegraphics[width=0.34\linewidth]{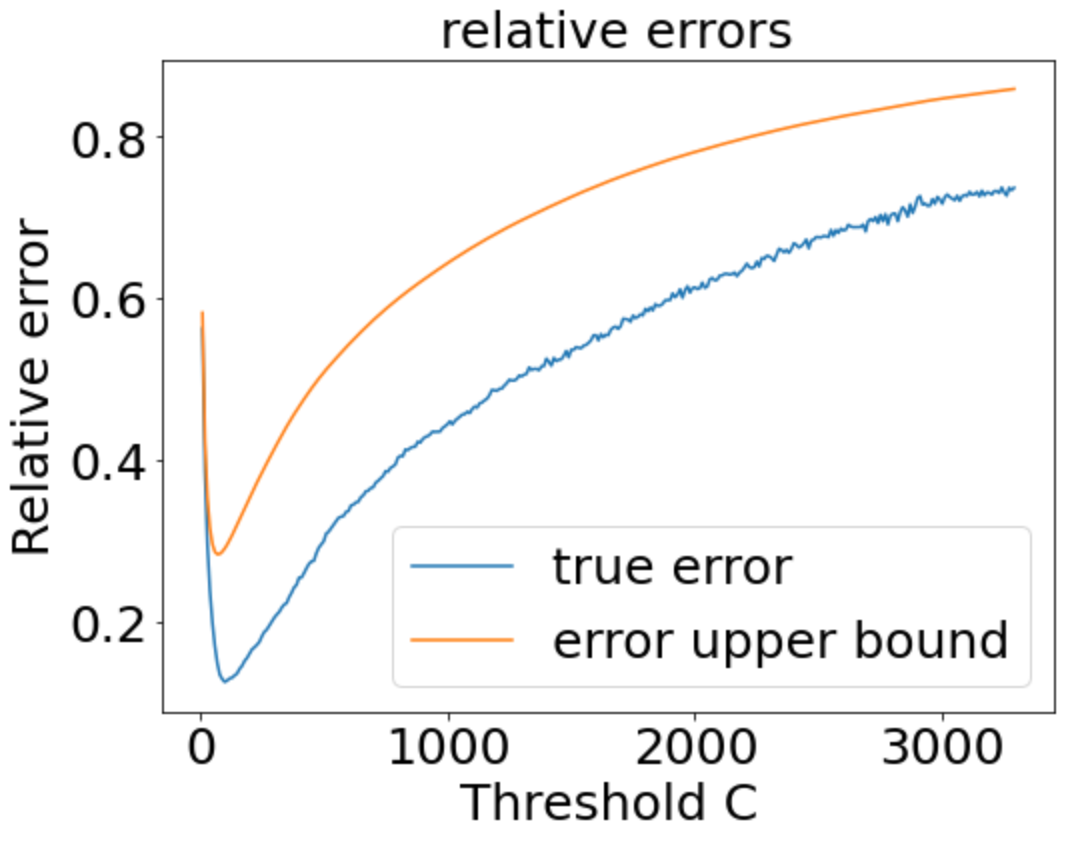} 
    \includegraphics[width=0.34\linewidth]{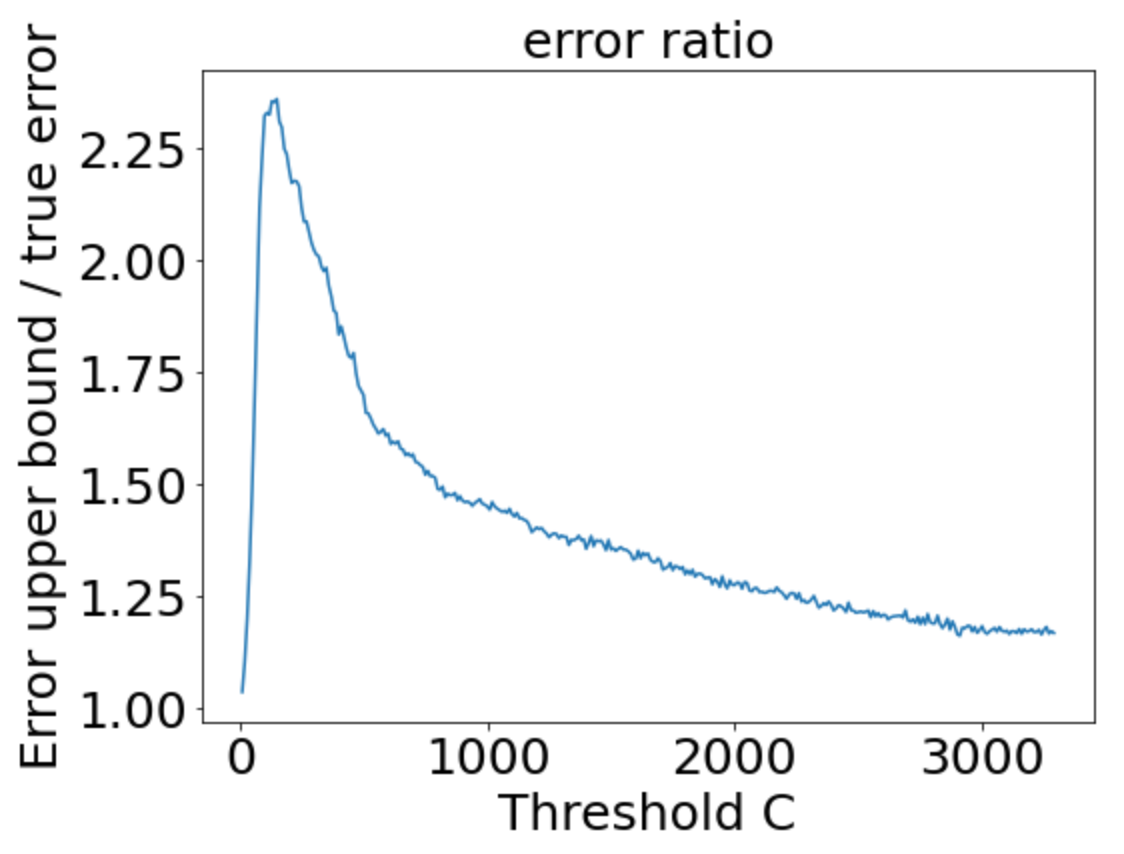}
    \caption{Error plots for different datasets for unbounded domains. \textbf{Top}: synthetic non-i.i.d data with $d=50$ (unknown). \textbf{Bottom:} Sentiment 140 with $d=1000$ (unknown). \textbf{Left column}: comparison of true error and predicted error upper bound. \textbf{Right column:} the orange/blue ratio in the left plots.}
    \label{fig:open-domain-loss}
\end{figure*}

\begin{table*}[t]
    \centering
        \def\arraystretch{1.13}
        \caption{Relative loss of the best threshold and Algorithm~\ref{alg:opt-threshold-open-domain} for unbounded domains.}
    \begin{tabular}{c c  c c c c }
    \hline
        Dataset & Support size $d$ (unknown) & Best $C$ & Private $C^*$ & Median & 90\% quantile\\
        \hline
        i.i.d. & $50$ & $0.001$ & $0.0015$ & 0.0397 & 0.0049 \\
        i.i.d. & $100$ & $0.0026$ &  $0.0026$& 0.0394 & 0.0053\\
        i.i.d. & $200$ & $0.0048$ &  $0.0048$ & 0.0397 & 0.0068\\\hline
        non-i.i.d. & $50$ & $0.0094$ & $0.0124$ &  0.6957 & 0.1645\\
        non-i.i.d.& $100$ & $0.0175$ &  $0.0224$ & 0.6957 & 0.1685\\
        non-i.i.d.& $200$ & $0.0461$ &  $0.0462$ & 0.7016 & 0.1713\\\hline
        Sent. 140& $100$ & $0.0076$ & $0.0285$ & 0.6027 & 0.2742\\
        Sent. 140& $1000$ & $0.1250$ & $0.1483$ & 0.5912 & 0.2699\\\hline
    \end{tabular}
    \label{tab:open-domain-performance}
\end{table*}

We run experiments on two real-world datasets: Sentiment140 \citep{sentiment140}, a twitter dataset that contains user tweets, and SNAP \citep{cho2011friendship}, a social network dataset that contains the location information of check-ins by users. For Sentiment140, we parse each user's tweets to words, and treat each word as an element. For SNAP, each element is a location, and each user has check-ins to multiple locations in the dataset. 

Since running on all elements (an order of $10^6$) is costly and and the error is ususally prohibitively large, we choose the top $d$ elements in the datasets and only run experiments on those. We measure the relative loss of $\widehat{N}$, 
\begin{equation}
    \frac{\sum^d_{j=1} |\HistSum_j - \widehat{N}_j|} { \|\HistSum\|_1}.
    \label{equ:relative-loss}
\end{equation}

\subsection{Bounded domain}

In all experiments, the privacy budget for estimating $C$ is $\varepsilon=0.1, \delta=1/2n$, and the budget for Algorithm~\ref{alg-clipping} is $\varepsilon=1, \delta=1/2n$. For DP-SGD with sparsity assumptions, we set $s = 0.1d$\footnote{The choice $s =0.1d$ is arbitrary (i.e. not a function of the underlying datasets) and has not been tuned.} and clip each $\|N_i\|_1/\|N_i\|_2$ to $\sqrt{s}$ when estimating $C^*$. This introduces bias when the assumption is not satisfied for some users. However if the percentage of such users is small, this effect can be negligible.

We evaluate different algorithms for estimating the clipping threshold $C$ for the Gaussian mechanism given in Algorithm~\ref{alg-clipping}.
We compare the performance of the following methods: $(i)$ $C^*$: The non-private clipping threshold   given in Theorem~\ref{thm:gaussian-2opt}. $(ii)$ \emph{DP-$M$-quantile}: 
inspired by the 2-approximation quantile in \citep{amin2019bounding}, we set $C$ to be
the $M^{th}$ largest value of $\|N_i\|_2$, where $M$ is given in Theorem~\ref{thm:gaussian-2opt}. We estimate it by gradient descent with differential privacy, e.g. \citet[Section 2]{andrew2021differentially}. This corresponds to a slightly different private version of the clipped-mean estimator in \citet[Section 3]{huang2021instance}. $(iii)$ $C_{\text{DPSGD}}$: estimation of $C^*$ with DPSGD algorithm (Corollary~\ref{cor:bassily}). and $(iv)$ $C_{\text{output}}$: estimation of $C^*$ with output pertubation (Algorithm~\ref{alg-convex}).


In Figure \ref{fig:hist-experiment}, we show the comparison of these threshold estimation algorithms with different choices of $d$ in $[100, 5000]$. The results with both datasets are similar, but SNAP has much higher errors, possibly because of the location information in SNAP is more non-i.i.d compared to the words in Sentiment140. Setting $C$ to $\emph{DP-$M$-quantile}$ according to \citep{amin2019bounding} and \cite{huang2021instance} does not have any theoretical support, and the errors are relatively high. For Algorithm~\ref{alg-convex}, we run experiments with $\Cmax = \emph{\text{DP-$M$-quantile}}$ and $\Cmax = 150$ (see Appendix~\ref{sec:cm15-experiment}). Of all the algorithms, $C_{\text{DPSGD}}$ has similar performance to the true $C^*$  without differential privacy.

\subsection{Unbounded domain}
In this section, we run experiments for Algorithm~\ref{alg:open-domain-algorithm} with threshold $C$ chosen by Algorithm~\ref{alg:opt-threshold-open-domain}.  We tested on both real and synthetic datasets. We set $\eps=1, \eps'=0.1$ and $\delta=1/(2n)$ where $n$ is the number of users in the respective datasets. 
We compared our algorithms to two non-private baselines where $C$ is the median and the 90\% quantile of the $\ell_1$ norms.  In this section $d$ should be interpreted as the actual support size of the aggregate histogram that is not known beforehand.  Our algorithm \textit{ does not} require prior knowledge of $d$. 

For Sentiment 140 we choose $d=100$ and $1000$ words and treat those as the support of the histograms. We generated synthetic datasets with $n=5\times 10^5$ users with both i.i.d. and non-i.i.d. data over support sizes $d=50$, $100$, or $200$.  Let $\bp$ be a discrete distribution over $[d]$ with probability mass proportional to $1/(j+50)$ for $j\in[d]$. In the i.i.d. setting, each user draws $\Poi(100)$ samples from $\bp$.  In the non-i.i.d. setting, let $\lambda_1, \ldots, \lambda_i\sim 100\Dir(2)$,  user $i$ draws $\Poi(\lambda_i)$ samples from $\bp_i\sim\Dir(\bp/2)$.  

We search for the best $C$ over $[10, 20,\ldots, 1500]$.
Figure~\ref{fig:open-domain-loss} compares the average error of Algorithm~\ref{alg:open-domain-algorithm} over 3 independent runs (blue) and the error upper bound in Algorithm~\ref{alg:opt-threshold-open-domain} (orange) on the non-i.i.d. synthetic dataset with $d=50$. We can see that~\eqref{equ:open-domain-upper-bound} indeed upper bounds the expected error, and their ratio is within $O(\log(\Cmax/2\delta))$.

Table~\ref{tab:open-domain-performance} compares the performance of the best $C$ (obtained non-privately) and the private estimate obtained by Algorithm~\ref{alg:opt-threshold-open-domain}. The performance is measured by relative loss defined by~\eqref{equ:relative-loss}. We can see that for all datasets, the private estimate is close to the performance achieved by the best threshold, and significantly outperforms the non-privately chosen median and 90\% quantiles. This further suggests the need to threshold according to the dataset instead of choosing a fixed threshold or quantile. Furthermore, focusing on the results for Sentiment 140 with $d=1000$ (i.e. the bottom row of Figure~\ref{fig:open-domain-loss} and Table~\ref{tab:open-domain-performance}), we observe that our algorithm yields good performance even when $C_m$ significantly overshoots the true optimal threshold, which demonstrates robustness against the choice of $C_m$.

\section{Conclusion}
\label{sec:conclusion}

We studied histogram estimation under user-level differential privacy in the heterogeneous scenario for bounded and unbounded domains. We proposed algorithms to choose the best user contribution bound that achieve 2-approximation and logarithmic approximation for bounded and unbounded domains respectively. We also showed that clipping bias introduced by bounding user contribution may be reduced under distribution assumptions. Finally, we empirically demonstrated the practicality of the proposed methods. 

\section{Acknowledgements}

The authors thank Alex Kulesza for helpful comments and discussions.

\newpage

\bibliographystyle{icml2023}
\bibliography{refs}

\newpage

\appendix
\onecolumn
\input{appendix}

%% file: appendix.tex
\section{Detailed proof for bounded domain algorithms}

\subsection{Proof of Theorem~\ref{thm:gaussian-2opt}}
\label{sec:gaussian-2opt-proof}
\begin{proof}
Recall that $\sigma$ is a function of $\eps, \delta$ and  $M=d\sigma\sqrt{\frac{2}{\pi}}$. We can upper bound the error as follows. 
\begin{align}
\mathbb{E}[\|\widehat{N} - \HistSum\|_1] &= \mathbb{E}\left[\left\|\sum_{i=1}^n \text{clip}_2(N_i, C) +\mathcal{N}(0,\mathbb{I}C^2\sigma^2) - \sum_{i=1}^n N_i \right\|_1\right] \nonumber \\ 
& \leq  \mathbb{E}\left[\left\|\sum_i \text{clip}_2(N_i, C)  - \sum_{i=1}^n N_i \right\|_1\right]  + \mathbb{E}[\|\mathcal{N}(0,\mathbb{I}C^2\sigma^2)\|_1] \nonumber \\ 
&=   \sum_{i: ||N_i||_2 > C} \left(1 - \frac{C}{||N_i||_2} \right) ||N_i||_1 + C \cdot M  \nonumber\\ 
&=\sum_{i=1}^n\max\left\{ 1 - \frac{C}{||N_i||_2}, 0\right\} ||N_i||_1  + C \cdot M = G(C).
\label{eq:upper_bound}
\end{align} 
Equation~\ref{eq:upper_bound} is convex with respect to $C$. To optimize the upper bound on the error, we will take the sub-derivative with respect to $C$ and set it to zero. This gives us the following equation
\begin{equation}
\label{eq:C}
\sum_{i: ||N_i||_2 > C} \frac{||N_i||_1}{||N_i||_2} = M.
\end{equation} 
Roughly we want to choose $C$ that satisfies the above equality. The precise value of $C^*$ is 
\[
C^*=\arg\min_{C\ge 0}\left\{\sum_{i:\|N_i\|_2>C}\frac{\|N_i\|_1}{\|N_i\|_2}\le M\right\}
\]
$C^*$ minimizes the right hand side of~\eqref{eq:upper_bound}, and it also makes the expected $\ell_1$ loss at most twice the loss of the optimal loss with this algorithm. Formally, suppose $Q$ is the $\ell_2$-norm that minimizes $\mathbb{E}[\|\widehat{N} - \HistSum\|_1]$. Let $\mathcal{Z}=[Z_1, \ldots, Z_d] \sim \mathcal{N}(0,I\sigma^2)$ and $\clip_2(N_i, Q)_j$ be the $j$ the coordinate of $\clip_2(N_i, Q)$, then:

\begin{align*}
& \mathbb{E}\left[\left\|\sum_i \text{clip}_2(N_i, Q) +\mathcal{Z}- \sum_i N_i \right\|_1\right] =\sum_{j=1}^d\mathbb{E}\left[\left|\text{clip}_2(N_i, Q)_j-N_{i, j}+Z_j\right|\right]\\
= & \sum_{j=1}^d\mathbb{E}\left[\left|\sum_i \text{clip}_2(N_i, Q)_j + Z_j - \sum_i N_{i, j} \right| \bigg| Z_j < 0 \right] \cdot \Pr(Z_j < 0)
\\
\quad&+ \sum_{j=1}^d\mathbb{E}\left[\left|\sum_i \text{clip}_2(N_i, Q)_j + Z_j - \sum_i N_{i, j} \right| \bigg| Z_j \ge 0 \right] \cdot \Pr(Z_j \ge 0)\\
\ge & \sum_{j=1}^d\mathbb{E}\left[\left|\sum_i \text{clip}_2(N_i, Q)_j + Z_j - \sum_i N_{i, j} \right| \bigg| Z_j < 0 \right] \cdot \Pr(Z_j < 0)\\
= & \frac{1}{2}\sum_{j=1}^d\mathbb{E}\left[\left|\sum_i \text{clip}_2(N_i, Q)_j + Z_j - \sum_i N_{i, j} \right| \bigg| Z_j < 0 \right] \\
= & \frac{1}{2} \left(Q \cdot M + \sum_{i: ||N_i||_2 > Q} \left(1 - \frac{Q}{||N_i||_2}\right) ||N_i||_1\right) \\
\ge & \frac{1}{2} \left(C^* \cdot M + \sum_{i: ||N_i||_2 > C^*} \left(1 - \frac{C^*}{||N_i||_2} \right) ||N_i||_1\right)=\frac{1}{2}G(C^*).
\end{align*}
This shows that $C^*$ yields a 2-approximation.
\end{proof}

\subsection{Proof of Corollary~\ref{cor:hist}}
To estimate $C^*$ privately, one can use the output perturbation algorithm. For ease of analysis we consider the regularized problem. More precisely, let 
$$f_i(C)=\max\{1-C/\|N_i\|_2, 0\}\|N_i\|_1.$$
Note that $f_i$ is $L$-Lipschitz where $L=\sqrt{d}$. The goal is to minimize the following function
\begin{equation}
    F_1(C)=\frac{1}{n}\sum_{i=1}^nf_i(C)+\frac{CM}{n}+\frac{\lambda}{2}C^2. 
\end{equation}
Let $C_1^*=\arg\min_{C\ge 0}F_1(C)$. We first compute the sensitivity of $C_1^*$ as a function of the dataset.

We first compute the sensitivity of $C'$. Consider a pair of neighboring datasets $D$ and $D'$ which only differ by the $n$th user. 
\begin{lemma}
\label{lem:th-sensitivity}
Let $N_n'$ be a histogram and $f_n'(C)$ defined similarly as $f_n(C)$ with $N_n$ replaced by $N_n'$. Let $ F_1(C)=\frac{1}{n}\sum_{i=1}^nf_i(C)+\frac{CM}{n}+\frac{\lambda}{2}C^2$ and 
$F_2(C)=\frac{1}{n}\sum_{i=1}^{n-1}f_i(C)+\frac{1}{n}f_n'(C)+\frac{CM}{n}+\frac{\lambda}{2}C^2$. Let $C_1^*=\arg\min_{C\ge 0} F_2(C)$ and $C_2^*=\arg\min_{C\ge 0} F_2(C)$. Then, 
\[
    |C_1^*-C_2^*|\le \Delta:=\frac{4\sqrt{s}}{\lambda n}, 
\]
\end{lemma}
\begin{proof}
Observe that $f_i(C)$ is $\sqrt{s}$ Lipschitz. Let $L = \sqrt{s}$.
\begin{align*}
    &\quad n(F_1(C_2^*)-F_1(C_1^*))\\
    &=\sum_{i=1}^nf_i(C_2^*)-\sum_{i=1}^nf_i(C_1^*)+M(C_2^*-C_1^*)+\frac{n\lambda}{2}((C_2^*)^2-(C_1^*)^2)\\
    &=\sum_{i=1}^{n-1}f_i(C_2^*)-\sum_{i=1}^{n-1}f_i(C_1^*)+M(C_2^*-C_1^*) +\frac{n\lambda}{2}((C_2^*)^2-(C_1^*)^2) +f_n(C_2^*)-f_n(C_1^*) \\
    &=n(F_2(C_2^*)-F_2(C_1^*))+f_n(C_2^*)-f_n(C_1^*) -(f_n'(C_2^*)-f_n'(C_1^*))\\ 
    &\le |f_n(C_2^*)-f_n(C_1^*)|+|f_n'(C_2^*)-f_n'(C_1^*)| \\
    &\le 2L|C_2^*-C_1^*| 
\end{align*}
Since $F_1$ is $\lambda$-strongly convex, we have
\[
F_1(C_2^*)-F_1(C_1*)\ge\frac{\lambda}{2}|C_2^*-C_1^*|^2
\]
Combining the two parts, 
\[
|C_2^*-C_1^*|\le\frac{4L}{\lambda n}
\]
\end{proof}

Now we can characterize performance of the combined algorithm which uses the output of Algorithm~\ref{alg-convex}, $\hat{C}$, as the clipping threshold in Algorithm~\ref{alg-clipping}.

\begin{lemma} 
\label{lem:hist-error}
Let $C_m$ be an upper bound on $C^*$. Then
\[
\EE[\mathcal{L}_G(\hat{C}, D)]-2\inf_{C\ge 0}\mathcal{L}_G(C, D)\le \frac{n\lambda C_m^2}{2}+\frac{4s}{\lambda \eps'}.
\]
\end{lemma}

\begin{proof}
Recall the definition of $C^*_1$ from Lemma~\ref{lem:th-sensitivity}.
First we write the expression,
\begin{align*}
 &  \EE[\mathcal{L}_G(\hat{C}, D)]-2\inf_{C>0}\mathcal{L}_G(C, D) \\
 &\le \EE[G(\hat{C})]-G(C^*)\\
   &\le \EE[M\hat{C}]+\EE\left[\sum_{i=1}^nf_i(\hat{C})\right]-G(C^*)\\
    &= C_1^* M + \EE\left[\sum_{i=1}^nf_i(\hat{C})-\sum_{i=1}^nf_i(C_1^*)\right]+\sum_{i=1}^nf_i(C_1^*) - C^* M -\sum_{i=1}^nf_i(C^*).
\end{align*}
The first inequality comes from the proof of Lemma~\ref{thm:gaussian-2opt}. We bound the terms separately,
\begin{align}
    \EE\left[\sum_{i=1}^nf_i(\hat{C})-\sum_{i=1}^nf_i(C_1^*)\right]
    &\le\EE\left|\sum_{i=1}^nf_i(\hat{C})-\sum_{i=1}^nf_i(C_1^*)\right|\nonumber\\
    &\le n\sqrt{s}\EE[|\hat{C}-C_1^*|]\nonumber\\
    &\le n\sqrt{s}\frac{\Delta}{\eps'}=\frac{4\sqrt{s}(\sqrt{s})}{\lambda \eps'}.\label{equ:err-noise}
\end{align}
The remaining terms are bounded using the following fact
\[
\sum^n_{i=1} f_i(C_1^*) + C^*_1M + n\frac{\lambda}{2}(C_1^*)^2\le  \sum^n_{i=1} f_i(C_*) + C^*M +  n\frac{\lambda C^2_*}{2}.
\]
Hence, 
\begin{equation}
    \label{equ:err-regularize}
    \sum^n_{i=1} f_i(C_1^*) + C^*_1M-\sum^n_{i=1} f_i(C_*) - C^*M \le \frac{n\lambda C^2_*}{2}\le \frac{n\lambda C_m^2}{2}.
\end{equation}
Combining equation~\ref{equ:err-noise} and~\ref{equ:err-regularize} yields the desired result.
\end{proof}
The proof of differential privacy follows from Lemma~\ref{lem:th-sensitivity} and the definition of Laplace mechanism. Setting $\lambda = \frac{2\sqrt{2s}}{C_m\sqrt{n\eps'}}$ in Lemma~\ref{lem:hist-error} yields the error.

\section{Unbounded domain histograms}
\label{sec:open-domain-proof}

Recall that $\widehat{N}$ is the output of Algorithm~\ref{alg:open-domain-algorithm} and the expected error is characterized by 
\begin{equation}
    \mathbb{E}[\|\widehat{N} - \HistSum\|_1]= \sum_{j: \HistSum_j > 0} \mathbb{E}\left[ \left | \widetilde{N}_j 1_{\widetilde{N}_j>t}  - \sum_{i}{N_i}_j \right |\right] 
    \label{eqn:open-domain-error}
\end{equation}

Obviously~\eqref{eqn:open-domain-error} depends on the choice of the threshold $C$. The error can be large if $C$ is too small or too large, so the goal of our work is to find the best choice of $C$. 

\ignore{We have the following upper bound for~\eqref{eqn:open-domain-err}

\begin{theorem}
Let $\HistSum$ be the true aggregate histogram and $\widehat{N}$ be the private estimate obtained by the open-domain algorithm. Then, for any constant $\alpha\in(0, 1)$
    \begin{align}
        \mathbb{E}[\|\widehat{N} - \HistSum\|_1]&\le  \sum_{j:\bh_j>t}\left(\frac{C}{\eps}+\frac{1}{2}e^{-\frac{\eps}{C}(\bh_j-t)}\left(t-\frac{C}{\eps}\right)\right)+\sum_{j:\bh_j\le t}\nonumber\\
        &\quad+\sum_{j:\bh_j\le t}\left(\left(1+\frac{1}{2\alpha}\right)\bh_j + \frac{1}{2}(2\delta)^{1-\alpha}C^{\alpha}\left(\frac{1}{\eps}+(1-\alpha)\left(1+\log\frac{C}{2\delta}\right)\right)\right) \nonumber\\
           &\quad+ \sum_{i=1}^n\max\left\{ \|N_i\|_1 - C, 0\right\} \label{thm:open-err-ub}
    \end{align}
\end{theorem}

We can also prove a lower bound for~\eqref{eqn:open-domain-err}
\begin{theorem}
Let $\HistSum$ be the true aggregate histogram and $\widehat{N}$ be the private estimate obtained by the open-domain algorithm. Then
    \begin{align}
        2\mathbb{E}[\|\widehat{N} - \HistSum\|_1]&\ge  \sum_{j:\bh_j>t}\left(\frac{C}{\eps}+e^{-\frac{\eps}{C}(\bh_j-t)}\left(t-\frac{C}{\eps}\right)\right)+\sum_{j:\bh_j\le t}\bh_j \nonumber\\
           &\quad+ \sum_{i=1}^n\max\left\{ \|N_i\|_1 - C, 0\right\} \label{thm:open-err-lb}
    \end{align}
\end{theorem}
}

\subsection{Approximation of estimation error}
\begin{theorem}
Let $\HistSum$ be the true aggregate histogram and $\widehat{N}$ be the private estimate obtained by the unbounded-domain algorithm. 
    \begin{align}
        \mathbb{E}[\|\widehat{N} - \HistSum\|_1\mid \bh_j]&=\Theta\Bigg ( \sum_{j:\bh_j>t}\left(\frac{C}{\eps}+e^{-\frac{\eps}{C}(\bh_j-t)}\left(t-\frac{C}{\eps}\right)\right)\nonumber\\
        &\quad+\sum_{j:\bh_j\le t}\left(\bh_j+\HistSum_j+e^{-\frac{\eps}{C}(t-\bh_j)}\left(|\HistSum_j-t|-\HistSum_j+\frac{C}{\eps}\right)\right) \nonumber\\
           &\quad+ \sum_{i=1}^n\max\left\{ \|N_i\|_1 - C, 0\right\} \Bigg) \label{thm:open-err-tight}
    \end{align}
\end{theorem}
\begin{proof}
First, we state a fact about exponential distributions. 
\begin{lemma}
    Let $X$ be an exponential distribution with rate $\nu$ (i.e., $X\ge 0$ and $\Pr[X\ge t]=e^{-\nu x}$) and $a\ge 0$, then
    \[
    \EE[X|0\le X\le a] \Pr[0\le X\le a]= \frac{1}{\nu} - e^{-\nu a}\left(a+\frac{1}{\nu}\right).
    \]
    \label{lem:exponential}
\end{lemma}
\begin{proof}
    Recall that $\EE[X]=\frac{1}{\nu}$. Due to the memoryless property of $X$, we have $\EE[X|X>a]=a+\frac{1}{\nu}$. Thus,
    \begin{align*}
    \EE[X]  &= \EE[X|0\le X\le a]\Pr[0\le X\le a] + \EE[X|X>a]\Pr[X>a]\\
    &= \EE[X|0\le X\le a]\Pr[0\le X\le a] + e^{-\nu a}\left(a+\frac{1}{\nu}\right)\\
    &= \frac{1}{\nu}.        
    \end{align*}
    Rearranging the terms proves the lemma.
    
\end{proof}

We first divide the summation into two parts based on if $Z_j \geq 0$ or $Z_j < 0$. For notational simplicity, all expectations in this section contain an implicit condition on $\bh_j$.
\begin{align*}
& \mathbb{E}[\|\widehat{N} - \HistSum\|_1\mid \bh_j] \nonumber\\
& = \sum_{j: \HistSum_j > 0} \mathbb{E}\left[ \left | \left(\bh_j + Z_j \right) 1_{\bh_j+Z_j>t}  - \sum_{i}{N_i}_j \right |\right]\\
& = \frac{1}{2} \left(\sum_{j: \HistSum_j > 0}\mathbb{E}\left[ |\left(\bh_j + Z_j \right) 1_{\bh_j+Z_j>t} - \bh_j|\big| Z_j<0 \right] + \sum_{i: ||N_i||_1 > C} \left(||N_i||_1 - C\right) 
\right)  (**)\\
& +\frac{1}{2} \left(\sum_{j: \HistSum_j > 0}\mathbb{E}\left[ |\left(\bh_j + Z_j \right) 1_{\bh_j+Z_j>t} - \HistSum_j|\big| Z_j\ge 0 \right] 
\right), (*)
\end{align*}
where to prove $(**)$, we use the fact that if $Z_j < 0$, then $\bar{h}_j + Z_j < \bar{h}_j < \HistSum_j$ and furthermore
\begin{align*}
\sum_j \HistSum_j - \bar{h}_j =
& \sum_j\sum_{i} N_{i,j} - h_{i,j} \\
& =
\sum_{i} \sum_j N_{i,j} - h_{i,j} \\
& =
\sum_i ||N_i||_1 - ||h_i||_1 \\
& = \sum_{i: ||N_i||_1 > C} \left(||N_i||_1 - C\right).
\end{align*}

We now bound the middle term. We use the fact that conditioned on $Z_j<0$, $|Z_j|=-Z_j$ is an exponential distribution with mean $C/\eps$. 
\begin{align}
   &\sum_{j: \HistSum_j > 0}\mathbb{E}\left[ |\left(\bh_j + Z_j \right) 1_{\bh_j+Z_j>t} - \bh_j|\big| Z_j<0 \right]\nonumber\\
    &=\sum_{j: \bar{h}_j > t}\mathbb{E}\left[ |\left(\bh_j + Z_j \right) 1_{\bh_j+Z_j>t} - \bh_j|\big| Z_j<0 \right] + \sum_{j: \bar{h}_j \leq t}\mathbb{E}\left[ |\left(\bh_j + Z_j \right) 1_{\bh_j+Z_j>t} - \bh_j|\big| Z_j<0 \right]\nonumber \\
      &=\sum_{j: \bar{h}_j > t}\mathbb{E}\left[ |\left(\bh_j + Z_j \right) 1_{\bh_j+Z_j>t} - \bh_j|\big| Z_j<0 \right] + \sum_{j:\bh_j\le t}\bh_j  \nonumber\\
      &=\sum_{j:\bh_j>t} \left(\bh_j\Pr[Z_j\le t-\bh_j|Z_j<0] + \EE[|Z_j|\mid t-\bh_j<Z_j<0]\Pr[Z_j> t-\bh_j|Z_j<0]\right) + \sum_{j:\bh_j\le t}\bh_j \label{eqn:left-side-intermediate} \\
    &=\sum_{j:\bh_j>t} \left(\bh_j e^{-\frac{\eps}{C}(\bh_j-t)} +\frac{C}{\eps}-\left(\bh_j - t+\frac{C}{\eps}\right) e^{-\frac{\eps}{C}(\bh_j-t)}\right) + \sum_{j:\bh_j\le t}\bh_j\nonumber\\
   &=\sum_{j:\bh_j>t} \left(\frac{C}{\eps}+\left(t-\frac{C}{\eps}\right) e^{-\frac{\eps}{C}(\bh_j-t)}\right) + \sum_{j:\bh_j\le t}\bh_j\label{eqn:left-side-final} 
\end{align}
From~\eqref{eqn:left-side-intermediate} we used Lemma~\ref{lem:exponential}.

\begin{align}
    (*)&=
    \sum_{j: \HistSum_j > 0}\mathbb{E}\left[ |\left(\bh_j + Z_j \right) 1_{\bh_j+Z_j>t} - \HistSum_j|\big| Z_j\ge 0 \right] \nonumber\\
  &   = \sum_{j: \bar{h}_j \geq t}\mathbb{E}\left[ |\left(\bh_j + Z_j \right) 1_{\bh_j+Z_j>t} - \HistSum_j|\big| Z_j\ge 0 \right] +  \sum_{j: \bar{h}_j < t}\mathbb{E}\left[ |\left(\bh_j + Z_j \right) 1_{\bh_j+Z_j>t} - \HistSum_j|\big| Z_j\ge 0 \right] \nonumber\\
    &   = \sum_{j: \bar{h}_j \geq t}\mathbb{E}\left[ |\bh_j + Z_j - \HistSum_j|\big| Z_j\ge 0 \right] +  \sum_{j: \bar{h}_j < t}\mathbb{E}\left[ |\left(\bh_j + Z_j \right) 1_{\bh_j+Z_j>t} - \HistSum_j|\big| Z_j\ge 0 \right] \nonumber\\
      &   = \sum_{j: \bar{h}_j \geq t}\mathbb{E}\left[ |\bh_j + Z_j - \HistSum_j|\big| Z_j\ge 0 \right] \nonumber\\
      &\quad +  \sum_{j: \bar{h}_j < t}\mathbb{E}\left[ |\left(\bh_j + Z_j \right)  - \HistSum_j|\big| Z_j\ge t - \bar{h}_j \right]\Pr[Z_j\ge t-\bh_j\mid Z_j\ge 0]
      + 
      \text{Pr}(Z_j < t - \bar{h}_j\mid Z_j\ge 0) \HistSum_j\nonumber\\
&    =\sum_{j;\bh_j>t}\left(\HistSum_j-\bh_j+\frac{C}{\eps}(2e^{-\frac{\eps}{C}(\HistSum_j-\bh_j)}-1)\right)\label{eqn:right-side-large-hj}\\
    &\quad +\sum_{j:\bh_j<t}\left(\HistSum_j(1-e^{-\frac{\eps}{C}(t-\bh_j)})+e^{-\frac{\eps}{C}(t-\bh_j)}\EE[|h_j+Z_j-N_j| | h_j+Z_j>t]\right)\label{eqn:right-side-small-hj}\\
    &=\sum_{j:\bh_j>t}\Theta(\HistSum_j-\bh_j+\frac{C}{\eps})+\sum_{j:\bh_j<t}\left(\HistSum_j(1-e^{-\frac{\eps}{C}(t-\bh_j)})+e^{-\frac{\eps}{C}(t-\bh_j)}\Theta(|N_j-t|+\frac{C}{\eps})\right)\label{eqn:righ-side-final}
\end{align}
~\eqref{eqn:right-side-large-hj} is due to when $\bh_j>t$, 
\begin{align*}
    &\mathbb{E}\left[ |\bh_j + Z_j - \HistSum_j|\big| Z_j\ge 0 \right]\\
    &=\Pr[Z_j>\bN_j-\bh_j\mid Z_j\ge0]\mathbb{E}\left[ \bh_j + Z_j - \HistSum_j\big| Z_j>\bN_j-\bh_j \right]\\
    &\quad +\Pr[Z_j\le\bN_j-\bh_j\mid Z_j\ge0]\mathbb{E}\left[ \bN_j-\bh_j -Z_j\big| 0\le Z_j\le\bN_j-\bh_j \right]\\
    &=\frac{C}{\eps}e^{-\frac{\eps}{C}(\bN_j-\bh_j)}+(\bN_j-\bh_j)(1-e^{-\frac{\eps}{C}(\bN_j-\bh_j)})-\left(\frac{C}{\eps}-e^{-\frac{\eps}{C}(\bN_j-\bh_j)}\left(\bN_j-\bh_j+\frac{C}{\eps}\right)\right)\\
    &=\frac{C}{\eps}\left(2e^{-\frac{\eps}{C}(\bN_j-\bh_j)}-1\right)+\bN_j-\bh_j
\end{align*}

The conditional expectation in~\eqref{eqn:right-side-small-hj} is evaluated as,
\begin{align*}
    &\EE[|h_j+Z_j-N_j| | h_j+Z_j>t]\\
    &=\EE[|h_j+Z_j-N_j+t-h_j||Z_j>0]\\
    &=\EE[|t+Z_j-N_j||Z_j>0]
    =\begin{cases}
        (\HistSum_j-t+\frac{C}{\eps}(2e^{-\frac{\eps}{C}(\HistSum_j-t)}-1), &t<\HistSum_j \\
        t-N_j+\frac{C}{\eps}, &t\ge \HistSum_j\\
    \end{cases}\\
    &=\Theta(|N_j-t|+\frac{C}{\eps}),
\end{align*}
The final equality is due to Lemma~\ref{lem:const-appx-flipped-exponential}. Combining all the parts leads to~\eqref{eqn:righ-side-final}. 
\begin{lemma}
    Let $a>0$ be constant. The function $g(x)=x+\frac{1}{a}(2e^{-a x}-1)=\Theta(x+\frac{1}{a})$ for $x\ge 0$.
    \label{lem:const-appx-flipped-exponential}
\end{lemma}
\begin{proof}
    It is obvious that $g(x)\le x+\frac{1}{a}$ since $2e^{-a x}-1\le 1$. It remains to prove $g(x)\ge c(x+1/a)$ for some constant $c$.
    
    Consider the function $h(x)=g(x)/(x+1/a)$. Then,
    \[
    h'(x)=2\frac{\frac{1}{a}-e^{-a x}(x+2/a)}{(x+1/a)^2}
    \]
    Since the numerator $\phi(x)=\frac{1}{a}-e^{-ax}(x+2/a)$ is monotonically increasing for $x\ge 0$, and $\phi(0)=-1/a$, $\phi(\infty)=1/a$, there must be a unique $\xi>0$ such that $\phi(\xi)=0$. Thus $h(x)$ is decreasing in $(0, \xi)$ and increasing in $(\xi, \infty)$. The minimum of $h$ is reached when $x=\xi$, with a minimum value of
    \[
    h(\xi)=1-\frac{2}{a\xi+2}>0
    \]

    Rearranging the terms in equation $\phi(x)=0$, we easily note that $a\xi$ is the unique solution to the equation $e^{-x}(x+2)-1=0$. Note that
    \[
    e^{-1}(1+2)>1, \quad e^{-2}(2+2)<1
    \]
    Therefore, we must have $a\xi>1$. Note that $1-2/(a\xi+2)$ increases with $\xi$, thus,
    \[
    h(\xi)\ge 1-\frac{2}{1+2}=\frac{1}{3}.
    \]
    This implies $g(x)\ge \frac{1}{3}(x+1/a)$. 
    
\end{proof}
Now we can combine~\eqref{eqn:left-side-final},~\eqref{eqn:righ-side-final} to prove the theorem.

\begin{align*}
& \mathbb{E}[\|\widehat{N} - \HistSum\|_1] \nonumber\\
& = \sum_{j: \HistSum_j > 0} \mathbb{E}\left[ \left | \left(\bh_j + Z_j \right) 1_{\bh_j+Z_j>t}  - \sum_{i}{N_i}_j \right |\right]\\
& = \frac{1}{2} \sum_{j:\bh_j>t} \left(\frac{C}{\eps}+\left(t-\frac{C}{\eps}\right) e^{-\frac{\eps}{C}(\bh_j-t)}\right) + \frac{1}{2}\sum_{j:\bh_j\le t}\bh_j + \frac{1}{2}\sum_{i: ||N_i||_1 > C} \left(||N_i||_1 - C\right) 
 \\
&\quad +\frac{1}{2} \sum_{j:\bh_j>t}\left(\frac{C}{\eps}\left(2e^{-\frac{\eps}{C}(\bN_j-\bh_j)}-1\right)+\bN_j-\bh_j\right)+\frac{1}{2}\sum_{j:\bh_j\le t}\left(\HistSum_j(1-e^{-\frac{\eps}{C}(t-\bh_j)})+e^{-\frac{\eps}{C}(t-\bh_j)}\Theta(|N_j-t|+\frac{C}{\eps})\right)\\
& = \frac{1}{2} \sum_{j:\bh_j>t} \left(\frac{C}{\eps}+\left(t-\frac{C}{\eps}\right) e^{-\frac{\eps}{C}(\bh_j-t)}+\Theta(\HistSum_j-\bh_j+\frac{C}{\eps})\right) + \frac{1}{2}\sum_{i: ||N_i||_1 > C} \left(||N_i||_1 - C\right) \\
&\quad+\frac{1}{2}\sum_{j:\bh_j\le t}\left(\bh_j+\HistSum_j(1-e^{-\frac{\eps}{C}(t-\bh_j)})+e^{-\frac{\eps}{C}(t-\bh_j)}\Theta\left(|\bN_j-t|+\frac{C}{\eps}\right)\right)\\
\end{align*}
From Lemma~\ref{lem:const-appx-flipped-exponential}, the $\Theta$ expressions are upper bounded by a factor of 1 and lower bounded by a factor of 1/3. Therefore, 
\begin{align*}
&\quad \frac{1}{2}\sum_{i=1}^n\max\left\{ \|N_i\|_1 - C, 0\right\} +  \frac{1}{2}\sum_{j:\bh_j>t}\left(\frac{C}{\eps}+e^{-\frac{\eps}{C}(\bh_j-t)}\left(t-\frac{C}{\eps}\right)\right)\\
&\quad +\frac{1}{6}\sum_{j:\bh_j\le t}\left(\bh_j+\HistSum_j+e^{-\frac{\eps}{C}(t-\bh_j)}\left(|\HistSum_j-t|-\HistSum_j+\frac{C}{\eps}\right)\right)\\
    &\le\EE[\|\widehat{N}-\bN\|_1]\le  \sum_{i=1}^n\max\left\{ \|N_i\|_1 - C, 0\right\} +  \sum_{j:\bh_j>t}\left(\frac{C}{\eps}+e^{-\frac{\eps}{C}(\bh_j-t)}\left(t-\frac{C}{\eps}\right)\right)\nonumber\\
        &\quad+\frac{1}{2}\sum_{j:\bh_j\le t}\left(\bh_j+\HistSum_j+e^{-\frac{\eps}{C}(t-\bh_j)}\left(|\HistSum_j-t|-\HistSum_j+\frac{C}{\eps}\right)\right) \nonumber
\end{align*}

\end{proof}

\begin{corollary} Assume that $\delta\le 1/8$, $C\ge 1$ and $C\le e^{\eps/3\delta}$. Then
\begin{align*}
& \frac{1}{2} \sum_{i=1}^n\max\left\{ \|N_i\|_1 - C, 0\right\}+  \frac{1}{12} \sum_{j: \bh_j < t} \HistSum_j + \frac{1}{2}\sum_{j: \bh_j > t} \frac{C}{\epsilon} \\
& \leq  \mathbb{E}[\|\widehat{N} - \HistSum\|_1\mid \bh_j] \leq \sum_{i=1}^n\max\left\{ \|N_i\|_1 - C, 0\right\} +
   \sum_{j: \bh_j < t} \HistSum_j + \sum_{j: \bh_j > t}  t
\end{align*}
\end{corollary}
\begin{proof}
For terms with $\bh_j>t$,
\[
\frac{C}{\epsilon}\leq \frac{C}{\eps}+e^{-\frac{\eps}{C}(\bh_j-t)}\left(t-\frac{C}{\eps}\right) \leq t = \frac{C}{\epsilon} \log \frac{C}{2 \delta}.
\]

For the terms with $\bh_j\le t$,
\begin{align*}
& \frac{\bN_j}{2}\le W:=\bh_j+\HistSum_j+e^{-\frac{\eps}{C}(t-\bh_j)}\left(|\HistSum_j-t|-\HistSum_j+\frac{C}{\eps}\right)\le 2\bN_j.  \\
\end{align*}

To prove this, we consider $\bN_j\ge t$ and $\bN_j<t$ separately.
\paragraph{1. $\bN_j\ge t$} Since $t\ge C/\eps$ and $0\le\bh_j\le \bN_j$,
\[
W = \bh_j+\bN_j+e^{-\frac{\eps}{C}(t-\bh_j)}\left(\frac{C}{\eps}-t\right)\le 2\bN_j.
\]
$W$ is concave with respect to $\bh_j$, thus the minimum of $W$ must occur at either $\bh_j=0$ or $\bh_j=t$. When $\bh_j=0$,
\begin{equation}
W = \bN_j+\frac{2\delta}{e^\eps}\left(t-\frac{C}{\eps}\right)\ge \frac{\bN_j}{2}.    \label{equ:lb-hj=0}
\end{equation}
as long as $\bN_j\ge \frac{4\delta}{Ce^{\eps}}(t-\frac{C}{\eps})$. Since $\delta\le 1/4$ and $C\ge 1$, we must have $\bN_j\ge t\ge\frac{4\delta}{Ce^{\eps}}(t-\frac{C}{\eps}) $, so the condition is satisfied.

When $\bh_j=t$, 
\[
\bN_j\le W=\bN_j+\frac{C}{\eps}\le 2\bN_j.
\]
The final inequality is due to $C/\eps\le t\le \bN_j$.

\paragraph{2. $\bN_j<t$}
\[
W = \bh_j+\bN_j+e^{-\frac{\eps}{C}(t-\bh_j)}\left(t-2\bN_j+\frac{C}{\eps}\right)
\]
If $\bN_j>(t+C/\eps)/2$, then
\[
\frac{\bN_j}{2}\le\bh_j+\bN_j+e^{-\frac{\eps}{C}(t-\bh_j)}\left(\frac{C}{\eps}-t\right)\le W\le \bh_j+\bN_j\le 2\bN_j
\]
The first inequality is proved similarly as~\eqref{equ:lb-hj=0}. 

If $\bN_j\le (t+C/\eps)/2$, then
\[
W\ge \bh_j+\bN_j\ge \bN_j.
\]
It remains to prove that $W\le (1+\log\frac{C}{2\delta})\bN_j$. We consider the following function 
\[
f(x) = x+\beta e^{ax}(\gamma - x), \quad a=\frac{\eps}{C}, \gamma = \frac{1}{2}\left(t+\frac{C}{\eps}\right), \beta=\frac{2\delta}{Ce^\eps}.
\]
Note that since $\bh_j\le \bN_j$, we have $W\le 2f(\bN_j)$. We just need to upper bound $g(x)=f(x)/x$ when $x\in [1, \gamma]$. Taking the derivative,
\[
g'(x)=-\beta e^{ax}\frac{ax^2-a\gamma x+\gamma}{x^2}.
\]
When $a\gamma\le 4$, $g'(x)\le 0$, thus 
\[
g(x)\le g(1) = 1+\frac{\delta}{e^{\eps(1-1/C)}}\left(\frac{1}{\eps}\left(1+\log\frac{C}{2\delta}\right)+1-\frac{1}{C}\right),
\]
which is at most 2 given the assumption that $C\le e^{\eps/3\delta}$.

When $a\gamma> 4$, then consider the root of $g'(x)=0$,
\[
x_1=\frac{\gamma}{2}\left(1-\sqrt{1-\frac{4}{\alpha\gamma}}\right), x_2=\frac{\gamma}{2}\left(1+\sqrt{1-\frac{4}{\alpha\gamma}}\right).
\]
$x_1$ is a local minimum, and $x_2$ is a local maximum of $g(x)$. Thus the maximum of $g(x)$ on $[1, \gamma]$ must be either $x=1$ or $x=x_2$. Note that $x_1< \gamma/2$ and $x_2<\gamma$. We evaluate and upper bound $g(x_2)$,
\begin{align*}
    g(x_2)&=1+\beta e^{ax_2}\left(ax_1-1\right)\\
    &\le 1+\beta e^{a\gamma}\frac{a\gamma}{2}\\
    &=1+\frac{2\delta}{Ce^\eps}\frac{1}{4}\left(1+\eps+\log\frac{C}{2\delta}\right)\sqrt{e\frac{e^\eps C}{2\delta}}\\
    &=1+\sqrt{\frac{2\delta}{Ce^\eps}}\frac{1}{4}\left(1+\eps+\log\frac{C}{2\delta}\right)\le 2
\end{align*}
Combining with the upper bound for $g(1)$ completes the proof.
\end{proof}

We simplify the above proof further to prove Theorem~\ref{thm:open-domain-error-main}.
\begin{corollary} Let $\Cmax$ be an upper bound on $C$.
\begin{align*}
& \frac{1}{2} \sum_{i=1}^n\max\left\{ \|N_i\|_1 - C, 0\right\}+  \frac{1}{ 12\log \frac{\Cmax}{2\delta}} \sum_j \min(h_j, t) \\
& \leq  \mathbb{E}[\|\widehat{N} - \HistSum\|_1\mid \bh_j] \leq 
 \\
& 2 \sum_{i=1}^n\max\left\{ \|N_i\|_1 - C, 0\right\} +
 \sum_j \min(\bar{h}_j, t).
\end{align*}
\end{corollary}
\begin{proof}
First observe that 
\begin{align*}
 \sum_{j: \bh_j < t} \HistSum_j  + \sum_{j: \bh_j > t} t
&=  \sum_{j: \bh_j < t} \bar{h}_j   + \sum_{j: \bh_j > t} t + \sum_{j: \bh_j < t}  (\HistSum_j - \bar{h}_j)\\
& \leq  \sum_{j: \bh_j < t} \bar{h}_j   + \sum_{j: \bh_j > t} t+ \sum_{j:\bN_j>0}  (\HistSum_j - \bar{h}_j)\\
& =  \sum_{j: \bh_j < t} \bar{h}_j   + \sum_{j: \bh_j > t} t  + \sum_{i}  \max( \|N_i\|_1 - C, 0)\\
& = \sum_j \min(\bar{h}_j, t) +  \sum_{i}  \max( \|N_i\|_1 - C, 0)
\end{align*}
Similarly, for the lower bound, observe that
\begin{align*}
\frac{1}{12}\sum_{j: \bh_j < t} \HistSum_j + \frac{1}{2}\sum_{j: \bh_j > t} \frac{C}{\epsilon} & \geq \frac{1}{12}\sum_{j: \bh_j < t} \HistSum_j + \sum_{j: \bh_j > t} \frac{t}{ 2\log \frac{\Cmax}{2\delta}} \\
& \geq \frac{1}{12 \log \frac{\Cmax}{2\delta}} \sum_j \min(\bh_j, t).
\end{align*}
\end{proof}


\section{Bias reduction}
We first note that in many datasets, counts of most symbols appear very few times. For example, in the Sentiment140 dataset, which contains counts for a total of roughly $6\cdot 10^5$ words distributed across $6\cdot 10^5$ users, the average counts of all words among the users are no more than two. Therefore we analyze the debiasing step when the $\lambda_i$'s are small and prove the following result for our desbiasing algorithm given in Algorithm~\ref{alg:poisson_single}.

\begin{theorem}
\label{thm:single_item_acc}
Suppose $N_i\sim \Poi(\lambda_i)$. Let $\lambdaAvg=\frac{1}{n}\sum_{i}\lambda_i$, $\Sigma=\frac{1}{n}\sum_{i=1}^n(\lambda_i-\lambdaAvg)^2$
and $\hat{\lambda}=\min\{1, \max\{0, \widehat{N}/n\}\}$,
where $\widehat{N}$ is the output of Algorithm~\ref{alg:poisson_single}. If $\lambdaAvg\le 1$, then 
    \begin{align}
    \EE[(\lambdaAvg-\hat{\lambda})^2]\le \gamma_C^2\left(\frac{C^2}{n^2\eps^2}+\frac{\lambdaAvg}{n}+\min\left\{1, \frac{1}{8\pi (C-1)}\right\}\Sigma^2\right).\label{equ:single_item_error}
    \end{align}
where $\gamma_C=\Pr[\Poi(1)< C]^{-1}\le \max\left\{e, \frac{1}{1-e^{-(C-1)^2/2C}}\right\}$.
If we further assume that $\lambda_i\le 1$, then
\begin{align}
    \EE[(\lambdaAvg-\hat{\lambda})^2]\le \gamma_C^2\left(\frac{C^2}{n^2\eps^2}+{\frac{\lambdaAvg}{n}}+\frac{1}{4 ((C-1)!)^2}\cdot\Sigma^2\right).
    \label{equ:single_item_all_less_than_1}
    \end{align}
\end{theorem}

The error consists of three terms. The first term is the error due to added noise, 
which is proportional to the clipping threshold $C$. 
The second term is essentially the variance of the random variable $\frac{1}{n}\sum_i N_i$, 
which is an inherent error due to the randomness in the counts $N_i$'s.
The third term is a bias term which depends on the closeness of user distributions, 
characterized by $\Sigma$, the variance of $\lambda_i$'s.
If the users' distributions are similar, then we can expect the estimation error to be small. Note that the bias term has a $1/C$ rate. The detailed proof is provided in Appendix~\ref{sec:single_item_acc_proof} We provide a bound with a better dependence on $C$ in Appendix~\ref{sec:single_item_improve}.

Next we analyze the the optimal choice of threshold $C$ after the debiasing step. Observe that to minimize the error upper bound in \eqref{equ:single_item_error}, roughly we want to choose
\[
C\propto\left(n\Sigma\right)^{2/3}+1.
\]
Therefore, when user's distribution are similar, we can use a smaller clipping threshold. This implies that as long as 
\[
\Sigma=O(n^{-1/4}),
\]
we can find a $\C$ that ensures a squared error of $O(1/n)$, which matches the error for the i.i.d. case. 

From~\eqref{equ:single_item_all_less_than_1}, when $\lambda_i\le 1$ for any user $i$, we can choose 
\[
\C\propto 1+\log\left(1+n\Sigma\right).
\]
This choice of $C$ always guarantees $O(1/n)$ error since when $\lambda_i\le1$ for all $i$, $\Sigma\le 1$.

In practice, we can also privately choose $\C$ as the top $\lceil1/\eps\rceil$ count as suggested by Lemma~\ref{lem:laplac-2opt}.

So far we have characterized the effect of debiasing under heterogeneous data. We next show that under mild assumptions debiasing helps even if data is non i.i.d.. The formal result is stated in~\cref{thm:single-poisson-gap}. The proof is in Appendix~\ref{sec:single-poisson-gap-proof}.

\begin{theorem}
\label{thm:single-poisson-gap}
Let $\bar{h}=\frac{1}{n}\sum_{i=1}^nh(\lambda_i)$. Let  $\hat{\lambda}_L=\widehat{N}_L/n$ be the average count obtained by Algorithm~\ref{alg-clipping} with Laplace noise. Assume that $\bar{h}\ge h_{\min}:=h(\lambdaAvg)-\frac{\lambdaAvg-h(\lambdaAvg)}{\gamma_C-1}$. Write $\bar{h}=h_{\min}+\alpha\frac{\lambdaAvg-h(\lambdaAvg)}{\gamma_C-1}$ where $\alpha\in(0, 1]$. Then the gap between Algorithm~\ref{alg-clipping} and Algorithm~\ref{alg:poisson_single} is
\begin{align}
    &\quad \EE[(\lambdaAvg-\hat{\lambda}_L)^2]-\EE[(\lambdaAvg-\hat{\lambda})^2]\nonumber\\
    &\ge \frac{\alpha(2\gamma_C-(\gamma_C+1)\alpha)}{\gamma_C-1}(\lambdaAvg- h(\lambdaAvg))^2 - O_C\left( \frac{1}{n} \right).\label{equ:debias-gap}
\end{align}
\end{theorem}

This implies that for any fixed $C$, under the assumptions stated in the theorem, with $n$ sufficiently large, 
there is always a constant gap between the two algorithms and debiasing helps even if the data is not i.i.d.. This result justifies the choice of $C$ as the optimal quantile suggested by Lemma~\ref{lem:laplac-2opt}.

We argue that the assumption of $\bar{h}\ge h_{\min}$ is not too restrictive.  It essentially requires that either $\lambda_i$'s are sufficiently similar, or $C$ is sufficiently larger than $\lambdaAvg$. Indeed, if $\lambda_i=\lambdaAvg$ for all user $i$, then $h(\lambdaAvg)=\bar{h}$; if $C$ is sufficiently large, then $h$ is  almost linear near $\lambdaAvg$ and hence $\bar{h}$ is close to $h(\lambdaAvg)$. 

As a specific example, set $C=2, \lambdaAvg=1$. If all $\lambda_i\in[0, 2]$,  due to concavity of $h$, we have $\bar{h}\ge h(2)/2\ge 0.729$. With some arithmetic, $h_{\min}\le 0.61$, and the first term in ~\eqref{equ:debias-gap} is at least 0.0217.  Note that this is the difference between squared errors. 
The gap between absolute errors could well be of order 0.1, which is significant considering that $\lambdaAvg=1$.  This example shows that Algorithm~\ref{alg:poisson_single} can achieve significant 
improvement even when the variance of $\lambda_i$s is constant.

\subsection{Proof of Theorem~\ref{thm:single_item_acc}}
\label{sec:single_item_acc_proof}
\begin{proof}

Let $Y=\frac{1}{n}Y_i$. Then, when $N_i\sim \Poi(\lambda_i)$, 

\begin{align*}
    \EE[Y] &=\frac{1}{n}\sum_{i=1}^nh(\lambda_i)\\
    &=\frac{1}{n}\sum_{i=1}^n \left(\sum_{j=0}^{\C-1}j\cdot \Pr[ N_i=j]+ \sum_{j=\C}^{\infty}C\cdot \Pr[ N_i=j]\right)\\
    &=\frac{1}{n}\sum_{i=1}^n\left(\C- \sum_{j=0}^{\C-1}\left(\C-j\right)\Pr[N_i=j]\right)\\
    &=\C-\frac{1}{n}\sum_{i=1}^n\sum_{j=0}^{\C-1}\left(\C-j\right)e^{-\lambda_i}\frac{\lambda_i^j}{j!}
\end{align*}

Let $Z=\Lap(\C/n\eps)$, then we have $\hat{\lambda}=h^{-1}(Y+Z)$.
We first bound the error for estimating $h(\lambda)$,
\begin{align}
    \EE\left[(h(\lambdaAvg)-h(\hat{\lambda}))^2\right]&=\EE\left[(h(\lambdaAvg)-\EE[Y]+\EE[Y]-h(\hat{\lambda}))^2\right]\nonumber\\
    &=\EE[(h(\hat{\lambda})-\EE[Y])^2]+(\EE[Y]-h(\lambdaAvg))^2\nonumber\\
    &= \EE[Z^2]+\EE[(Y-\EE Y)^2]+(\EE[Y]-h(\lambdaAvg))^2\nonumber\\
    &\le \frac{\C^2}{n^2\eps^2}+\EE[(Y-\EE Y)^2]+(\EE[Y]-h(\lambdaAvg))^2.\label{eq:h_error}
\end{align}

We first bound $\EE[(Y-\EE[Y])^2]$. 
\begin{align}
    \EE[(Y-\EE[Y])^2] &= \frac{1}{n^2}\EE\left[\left(\sum_{i=1}^n(Y_i-\EE[Y_i])\right)^2\right]\nonumber\\
    &=\frac{1}{n^2}{\sum_{i=1}^n \Var[Y_i]}\label{eq:clip_var}\\
    &\le \frac{1}{n^2}{\sum_{i=1}^n\Var[N_i]}=\frac{1}{n^2}{\sum_{i=1}^n\lambda_i}.\nonumber
\end{align}
The final inequality is due to $\Var[Y_i]\le \Var[N_i]$. To see this we use the symmetrization trick. Let $N_i'$ be an independent copy of $N_i$ and $Y_i'=\clip(N_i', C)$. Then by definition $|Y_i-Y_i'|\le |N_i-N_i'|$. Hence,
\begin{align*}
    \Var[Y_i]=\EE[(Y_i-Y_i')^2/2]\le \EE[(N_i-N_i')^2/2]=\Var[N_i].
\end{align*}

We then bound $|\EE[Y]-h(\lambdaAvg)|$. We compute Taylor expansion of $\EE[Y]$ at $\lambdaAvg$ with the Lagrangian remainder,
\begin{align*}
    \EE[Y]&= \frac{1}{n}\sum_{i=1}^n\left(h(\lambdaAvg)+h'(\lambdaAvg)(\lambda_i-\lambdaAvg)+\frac{h''(\xi_i)}{2}(\lambdaAvg-\lambda_i)^2 \right)\\
    &=h(\lambdaAvg)+\frac{1}{n}\sum_{i=1}^n\frac{h''(\xi_i)}{2}(\lambdaAvg-\lambda_i)^2,
\end{align*}
where $\xi_i\in(\min\{\lambdaAvg, \lambda_i\}, \max\{\lambdaAvg, \lambda_i\})$ . We move on to compute $h''(\lambda)$,
\begin{align*}
    h''(\lambda)&=e^{-\lambda}\Bigg(-\sum_{j=0}^{\C-1}\frac{\C-j}{j!}\lambda^j+2\sum_{j=0}^{\C-2}\frac{\C-(j+1)}{j!}\lambda^j -\sum_{j=0}^{\C-3}\frac{\C-(j+2)}{j!}\lambda^j\Bigg)\\
    &=-e^{-\lambda}\frac{\lambda^{\C-1}}{(\C-1)!}
\end{align*}

We can verify that for $C\ge 1$, $|h''(\lambda)|$ increases when $\lambda\le C-1$ and decreases when $\lambda\ge C-1$. Therefore by Sterling's approximation, 
\begin{align*}
    |h''(\xi_i))|\le \frac{(C-1)^{C-1}}{e^{C-1}(C-1)!}\le \frac{1}{\sqrt{2\pi(C-1)}}.
\end{align*}
Thus, 
\begin{align*}
    |\EE[Y]-h(\lambdaAvg)|\le \frac{1}{2\sqrt{2\pi (C-1)}}\cdot\frac{1}{n}\sum_{i=1}^n(\lambda_i-\lambdaAvg)^2.
\end{align*}

Combining all the parts we have
\begin{align*}
\EE[|h(\hat{\lambda})-h(\lambdaAvg)|]\le \frac{C}{n\eps}+\frac{1}{n}\sqrt{\sum_{i=1}^n\lambda_i}+\frac{1}{2\sqrt{2\pi (C-1)}}\cdot\frac{1}{n}\sum_{i=1}^n(\lambda_i-\lambdaAvg)^2.
\end{align*}

To bound the estimation error $\EE[|\lambdaAvg-\hat{\lambda}|]$, we first compute $h'(\lambda)$,
\begin{align*}
    h'(\lambda) &=-e^{-\lambda}\left(\sum_{j=0}^{\C-2}\frac{\C-(j+1)}{j!}\lambda^j-\sum_{j=0}^{\C-1}\frac{\C-j}{j!}\lambda^j\right)\\
    &=e^{-\lambda}\sum_{j=0}^{\C-1}\frac{\lambda^j}{j!}
\end{align*}

Let $X\sim \Poi(1)$. We note that for $\lambda\in[0, 1]$,
\[
|h'(\lambda)|\ge |h'(1)| = 1 - \Pr[X \ge C]=\frac{1}{\gamma_C} \ge 1-e^{-\frac{(C-1)^2}{2C}}.
\]
Hence, we can proceed to bound the error for estimating $\lambdaAvg$,
\begin{align}
    &\quad\EE[(\lambdaAvg-\hat{\lambda})^2]\nonumber\\
    &\le \sup_{\lambda\in[0, 1]}\frac{1}{|h'(\lambda)|^2} \EE[(h(\lambdaAvg)-h(\hat{\lambda}))^2]\label{eq:lamb_error_from_h_error}\\
    &\le  \gamma_C^2\bigg(\frac{C^2}{n^2\eps^2}+\frac{1}{n^2}{\sum_{i=1}^n\lambda_i}+\frac{1}{8\pi (C-1)}\cdot\left(\frac{1}{n}\sum_{i=1}^n(\lambda_i-\lambdaAvg)^2\right)^2\Bigg).\label{equ:1item_err}
\end{align}

If we assume that $\lambda_i\le1$ for all $i$ we can bound $|h''(\xi_i)|$ as
\begin{align*}
    |h''(\xi_i))|\le \max_{\lambda\in[0,1]}|h''(\lambda)|\le \frac{1}{(C-1)!}. 
\end{align*}
Hence the last term can be bounded as
\begin{align*}
    |\EE[Y]-h(\lambda)|\le \frac{1}{2(C-1)!}\cdot\frac{1}{n}\sum_{i=1}^n(\lambda_i-\lambda)^2.
\end{align*}
\end{proof}

\subsection{Improved bound of Theorem~\ref{thm:single_item_acc}}
\label{sec:single_item_improve}
\begin{theorem}
    Suppose $\lambda_i$ are i.i.d. drawn from a distribution $F$ with mean $\lambdaAvg$ and variance $\Sigma$. Then for $C\ge 3$, Algorithm~\ref{alg:poisson_single} yields an error of
    \[
\EE[(\lambdaAvg-\hat{\lambda})^2]\le \gamma_C^2\left(\frac{C^2}{n^2\eps^2}+\frac{\lambdaAvg}{n}+ \frac{1}{8\pi (C-1)}\left(0.83^{C-1}\Sigma+\Pr_{\lambda\sim F}\left(\lambda\ge \frac{C-1}{2}\right)\right)^2\right)
    \]
\end{theorem}
The theorem shows that if the distribution of user's average counts are concentrated and has small tail probability, then we can obtain small estimation error.

Suppose that $C\ge 3$. Let $\eta=1/2$. We bound $|\EE[Y]-h(\lambdaAvg)|$ by analyzing $\lambda_i$ close to $C-1$ and far from $C-1$ separately.
\begin{align*}
    &\quad|\EE[Y]-h(\lambdaAvg)|\\
    &=\frac{1}{n}\sum_{i=1}^n\frac{|h''(\xi_i)|}{2}(\lambdaAvg-\lambda_i)^2\\
    &=\frac{1}{n}\sum_{i:\frac{\lambda_i}{C-1}\in (\eta, 1/\eta)}\frac{|h''(\xi_i)|}{2}(\lambdaAvg-\lambda_i)^2 + \frac{1}{n}\sum_{i:\frac{\lambda_i}{C-1}\notin (\eta , 1/\eta)}\frac{|h''(\xi_i)|}{2}(\lambdaAvg-\lambda_i)^2
\end{align*}
If $\lambda_i\le \eta (C-1)$ and $C\ge 3$, then we also have $\xi_i\le \eta (C-1)$. In this case,
\begin{align*}
    &\quad|h''(\xi_i)|\le |h''(\eta (C-1))|=\frac{(\eta(C-1))^{C-1}}{e^{\eta(C-1)}(C-1)!}\\
    &\le \frac{1}{\sqrt{2\pi(C-1)}}\left(\frac{\eta e}{e^\eta}\right)^{C-1}\le \frac{1}{\sqrt{2\pi(C-1)}}\cdot 0.83^{C-1}
\end{align*}

The last inequality follows by $ex /e^x$ is increasing for $x\in[0, 1]$. Using a similar argument, we can verify that the above inequality also holds when $\lambda_i\ge (C-1)/\eta$. Therefore, 
\begin{align*}
|\EE[Y]-h(\lambdaAvg)|\le \frac{1}{2\sqrt{2\pi (C-1)}}\cdot\frac{1}{n}\Bigg(\sum_{i:\frac{\lambda_i}{C-1}\notin (1/2, 2)}0.83^{C-1}(\lambda_i-\lambdaAvg)^2 +\sum_{i:\frac{\lambda_i}{C-1}\in (1/2, 2)}(\lambda_i-\lambdaAvg)^2\Bigg).
\end{align*}

For convenience let $\Sigma=\frac{1}{n}\sum_{i=1}^n(\lambda_i-\lambdaAvg)^2$.
We can further bound the bias by
\begin{equation}
    |\EE[Y]-h(\lambdaAvg)|\le0.83^{C-1}\Sigma +\frac{\sqrt{2}(C-1)^{3/2}}{\sqrt{\pi }}\frac{1}{n}\sum_{i=1}^n\indic{\lambda_i>(C-1)/2}.
\label{equ:single_item_refined} 
\end{equation}

We can see that the bias term depends on $C$, $\Sigma$, and  $\sum_{i=1}^n\indic{\lambda_i>2(C-1)}$ (which depends on the distribution of $\{\lambda_i\}_{i=1}^n$). To ensure an $O(1/n)$ rate, the two terms should be at most $O(1/\sqrt{n})$,
\[
0.83^{C-1}\Sigma\le\frac{1}{\sqrt{n}}  \quad \frac{4(C-1)^{3/2}}{2\sqrt{2\pi }}\frac{1}{n}\sum_{i=1}^n\indic{\lambda_i>(C-1)/2}\le\frac{1}{\sqrt{n}}.
\]

Note that in the worst case
\[
\frac{1}{n}\sum_{i=1}^n\indic{\lambda_i>(C-1)/2}\le O\left(\frac{\Sigma}{(C-1)^2}\right),
\]
which recovers~\eqref{equ:single_item_error}. The bound could have a better dependence on $C$ if $\lambda_i$s are more concentrated. For example, if $\lambda_i\le 1$, then the above quantity is 0 as long as $C> 3$, and we can choose $C=3+O(\log (1+n\Sigma))$ to achieve $O(1/n)$ mean squared error.

More generally, if  $\lambda_i$s are from a distribution with exponential tail, i.e. $\Pr[\lambda_i\ge x]=O(\exp(-\Omega(x)))$, then
\[
\frac{1}{n}\sum_{i=1}^n\indic{\lambda_i>2(C-1)}\simeq \Pr[\lambda_i\ge (C-1)/2]=O(\exp(-\Omega(C-1))).
\]

Choosing $C=3+O(\log n)$ gives $O(1/n)$ mean squared error.





\subsection{Proof of Theorem~\ref{thm:single-poisson-gap}}
\label{sec:single-poisson-gap-proof}
Before we proceed to the proof, we first characterize the error of Algorithm~\ref{alg-clipping} with Laplace noise.

\begin{lemma}
Let $\hat{\lambda}_L=\widehat{N}_L/n$ where $\widehat{N}_L$ is the output of Algorithm~\ref{alg-clipping} with Laplace noise. Then
\[
\EE\left[(\lambdaAvg-\hat{\lambda}_L)^2\right]=\left(\lambdaAvg-\frac{1}{n}\sum_{i=1}^n h(\lambda_i)\right)^2+\frac{1}{n^2}\sum_{i=1}^n\Var(Y_i)+\frac{C^2}{n^2\eps^2}.
\]
\end{lemma}
\begin{proof}
\begin{align*}
    \EE\left[(\lambdaAvg-\hat{\lambda}_L)^2\right]
    &=\EE\left[(\lambdaAvg-\EE[\hat{\lambda}_L]+\EE[\hat{\lambda}_L]-\hat{\lambda}_L)^2\right]\\
    &=\EE\left[(\lambdaAvg-\EE[\hat{\lambda}_L])^2\right]+\EE\left[(\EE[\hat{\lambda}_L]-\hat{\lambda}_L)^2\right]\\
    &=\left(\lambdaAvg-\frac{1}{n}\sum_{i=1}^n h(\lambda_i)\right)^2+\frac{1}{n^2}\sum_{i=1}^n\Var(Y_i)+\frac{C^2}{n^2\eps^2}.
\end{align*}
\end{proof}

Now we have all the ingredients to complete Theorem~\ref{thm:single-poisson-gap}
\begin{proof}
Combining ~\eqref{eq:h_error},~\eqref{eq:clip_var}, and~\eqref{eq:lamb_error_from_h_error} we have
\begin{align*}
   \EE[(\lambdaAvg-\hat{\lambda})^2]\le\gamma_C^2\left(  \frac{C^2}{n^2\eps^2}+\frac{1}{n^2}\sum_{i=1}^n\Var[Y_i]+\left(h(\lambdaAvg)-\frac{1}{n}\sum_{i=1}^nh(\lambda_i)\right)^2\right).
\end{align*}

Let $\bar{h}=-\frac{1}{n}\sum_{i=1}^nh(\lambda_i)$.Therefore, 
\begin{align*}
    &\quad \EE[(\lambdaAvg-\hat{\lambda}_L)^2]-\EE[(\lambdaAvg-\hat{\lambda})^2]\\
    &= -(\gamma_C^2-1)\left(\frac{C^2}{n^2\eps^2}+\frac{1}{n^2}\sum_{i=1}^n\Var[Y_i]\right)+\left(\lambdaAvg-\bar{h}\right)^2-\gamma_C^2\left(h(\lambdaAvg)-\bar{h}\right)^2\\
\end{align*}
We bound the terms separately. Note that $\Var[Y_i]\le \Var[N_i]=\lambda_i$. Hence, 
\begin{align*}
    (\gamma_C^2-1)\left(\frac{C^2}{n^2\eps^2}+\frac{1}{n^2}\sum_{i=1}^n\Var[Y_i]\right)\le (\gamma_C^2-1)\left(\frac{C^2}{n^2\eps^2}+\frac{1}{n}\right).
\end{align*}
Recall the assumption that $\bar{h}\ge h_{\min}:=h(\lambdaAvg)-\frac{\lambdaAvg-h(\lambdaAvg)}{\gamma_C-1}$, and write $\bar{h}=h_{\min}+\alpha \frac{\lambdaAvg-h(\lambdaAvg)}{\gamma_C-1}$. The remaining part simplifies to 
\begin{align*}
    (\lambdaAvg-\bar{h})^2-\gamma_C^2(h(\lambdaAvg)-\bar{h})^2&=\frac{(\gamma_C-\alpha)^2-\gamma_C^2(1-\alpha)^2}{(\gamma_C-1)^2}(\lambdaAvg-h(\lambdaAvg))^2\\
    &=\frac{\alpha(2\gamma_C-(\gamma_C+1)\alpha)}{\gamma_C-1}(\lambdaAvg-h(\lambdaAvg))^2
\end{align*}
Combining the two parts completes the proof.
\end{proof}

\subsection{Experiments}
In this section, we run experiments for Algorithm~\ref{alg:poisson_single} with both synthetic datasets and words in Sentiment140. For the synthetic part, we generate $n=10^6$ users with $\lambda_1, \dots, \lambda_n$ from a Dirichlet distribution with parameter $\alpha$. Larger $\alpha$ means that the $\lambda_i$s are closer. In this experiment, we set $C$ to the privately estimated top $1/\varepsilon$ count among users, as discussed in ~\citet{amin2019bounding}.
Figure~\ref{fig:single_synthetic} shows that the debiased output of  Algorithm~\ref{alg:poisson_single} greatly reduces the error compared to the original output of Algorithm~\ref{alg-clipping}, especially for large $\alpha$ (meaning the dataset is more i.i.d. like).

\begin{figure}[t]
{
\centering
\includegraphics[width=0.35\textwidth]{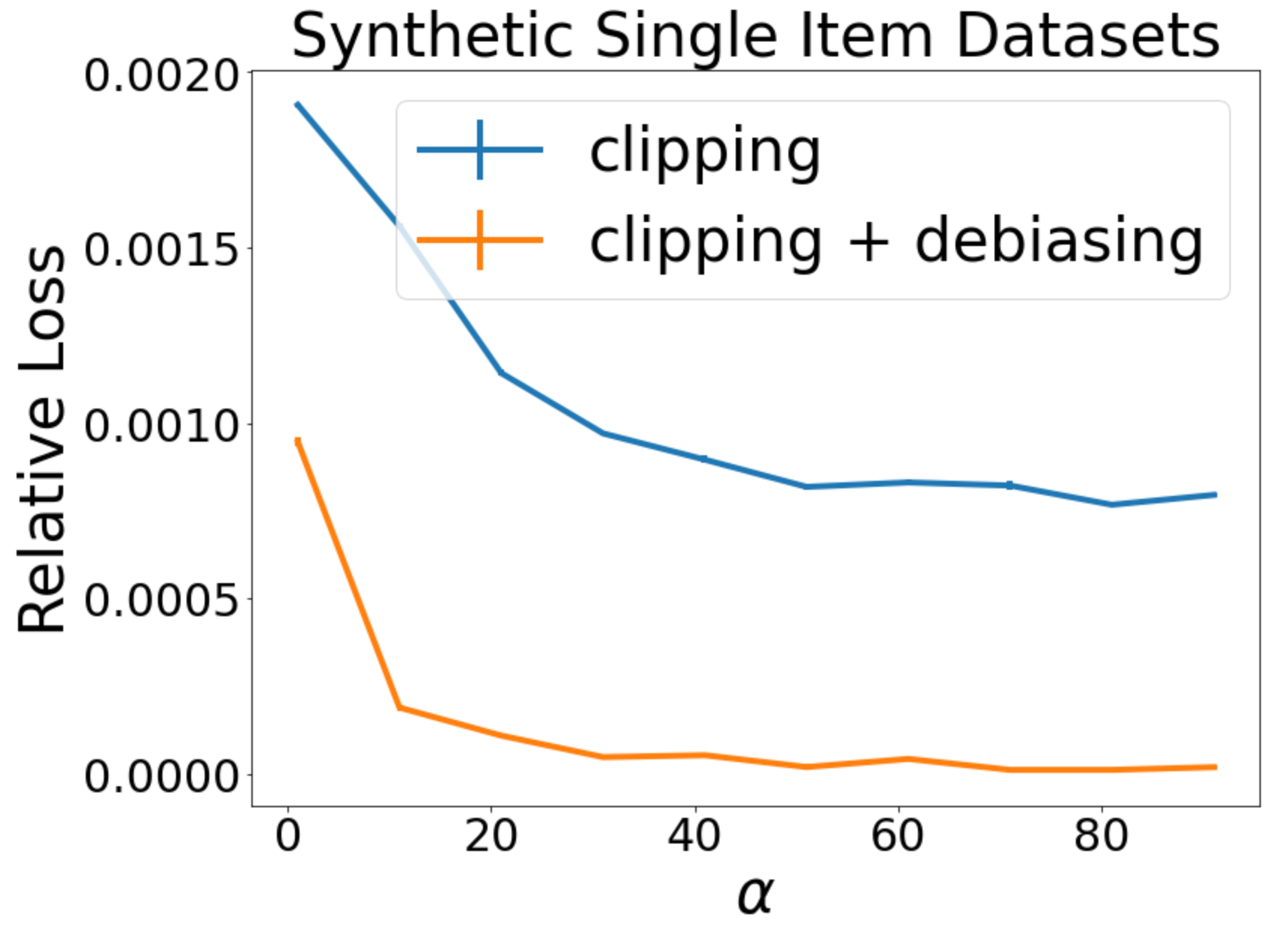}
\caption{Total counts estimation for single item on synthetic Poisson datasets. Larger $\alpha$ means that user data distributions are more similar.}
\label{fig:single_synthetic}
}
\end{figure}

We also run experiments on three population words in Sentiment140: ``the'', ``today'' and ``you''. Table~\ref{table:single_sentiment_140} shows that Algorithm~\ref{alg:poisson_single} performs better than Algorithm~\ref{alg-clipping}, but the gain is not as much as synthetic datasets that are close to i.i.d. distributions. 

\begin{table}
\centering  
\caption{Single word in Sentiment140. The ``debiasing'' columns are results of clipping + debiasing.}
  \begin{tabular}{  c  c  c  }
    \toprule
    & clipping  & debiasing \\
    word & avg loss$\pm$std & avg loss$\pm$std \\
    \midrule
    the & $0.0289$  & $0.0257$  \\ 
    & $\pm 2.40 \cdot 10^{-5}$ & $\pm2.48 \cdot 10^{-5}$ \\ \hline
    today & $0.0381$  & $0.0155$  \\ 
    & $\pm4.34 \cdot 10^{-5}$& $\pm3.34 \cdot 10^{-5}$\\ \hline
    you & $0.0745$  & $0.0671$ \\ 
    & $\pm2.83 \cdot 10^{-5}$ & $\pm7.26 \cdot 10^{-5}$ \\
    \bottomrule
  \end{tabular}
  \label{table:single_sentiment_140}
\end{table}

\subsection{Extension to $d>1$} We now discuss two possible extensions to the general $d$.

1. A natural extension to the entire histogram is to apply Algorithm~\ref{alg:poisson_single} to each symbol in the histogram separately. To ensure $(\eps, \delta)$ differential privacy, we assign each symbol a privacy budget of $O(\eps/\sqrt{d\log(1/\delta)})$ by strong composition~\citep{kairouz2017composition}. The main disadvantage is that when $d$ is large, clipping each coordinate separately may perform poorly compared to clipping the $\ell_1$ or $\ell_2$ norm of the entire histogram. 

2. We can generalize Algorithm~\ref{alg:poisson_single} to $d>1$ by replacing 1-d clipping with the high dimensional clipping functions as defined in Algorithm~\ref{alg-clipping}. Then choose a suitable function $g$ that essentially inverts the expectation of the clipped vector $Y_i$. However finding such inverse may be difficult in high dimensions as it likely involves non-convex optimization.

\section{Additional experiments}
\label{sec:cm15-experiment}
\subsection{Algorithm~\ref{alg-convex} with fixed $\Cmax=150$}

Fixing $\Cmax$ performs better than $\Cmax = \emph{\text{DP-$M$-quantile}}$, because in the latter case, we need to split half of the privacy budget to estimate \emph{\text{DP-$M$-quantile}}. We note that the output perturbation method has a large variance for the SNAP dataset and investigating the reason behind is an interesting future direction.

\begin{figure}[h]
    \centerline{\includegraphics[width=0.38\textwidth]{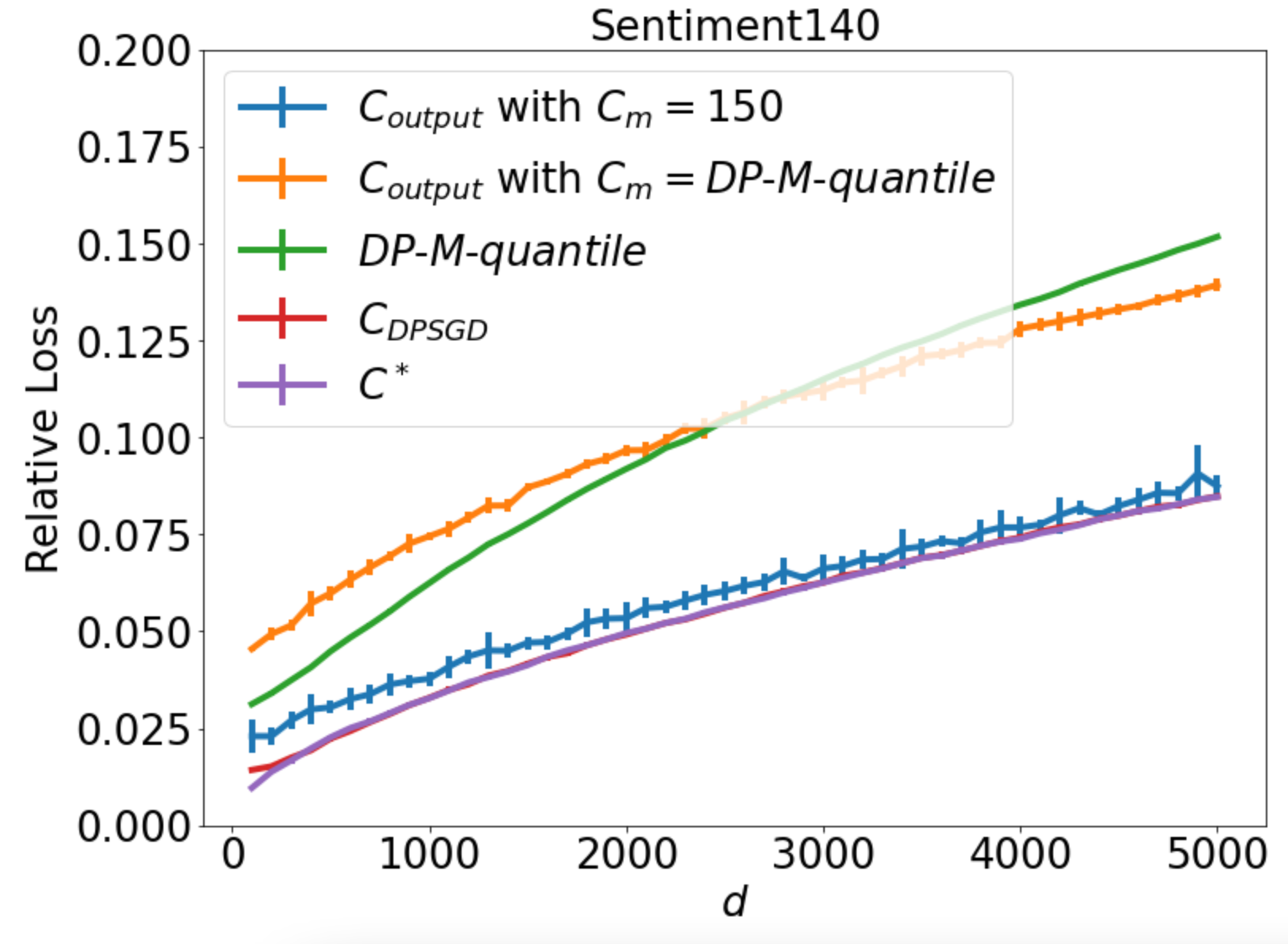}\quad\includegraphics[width=0.38\textwidth]{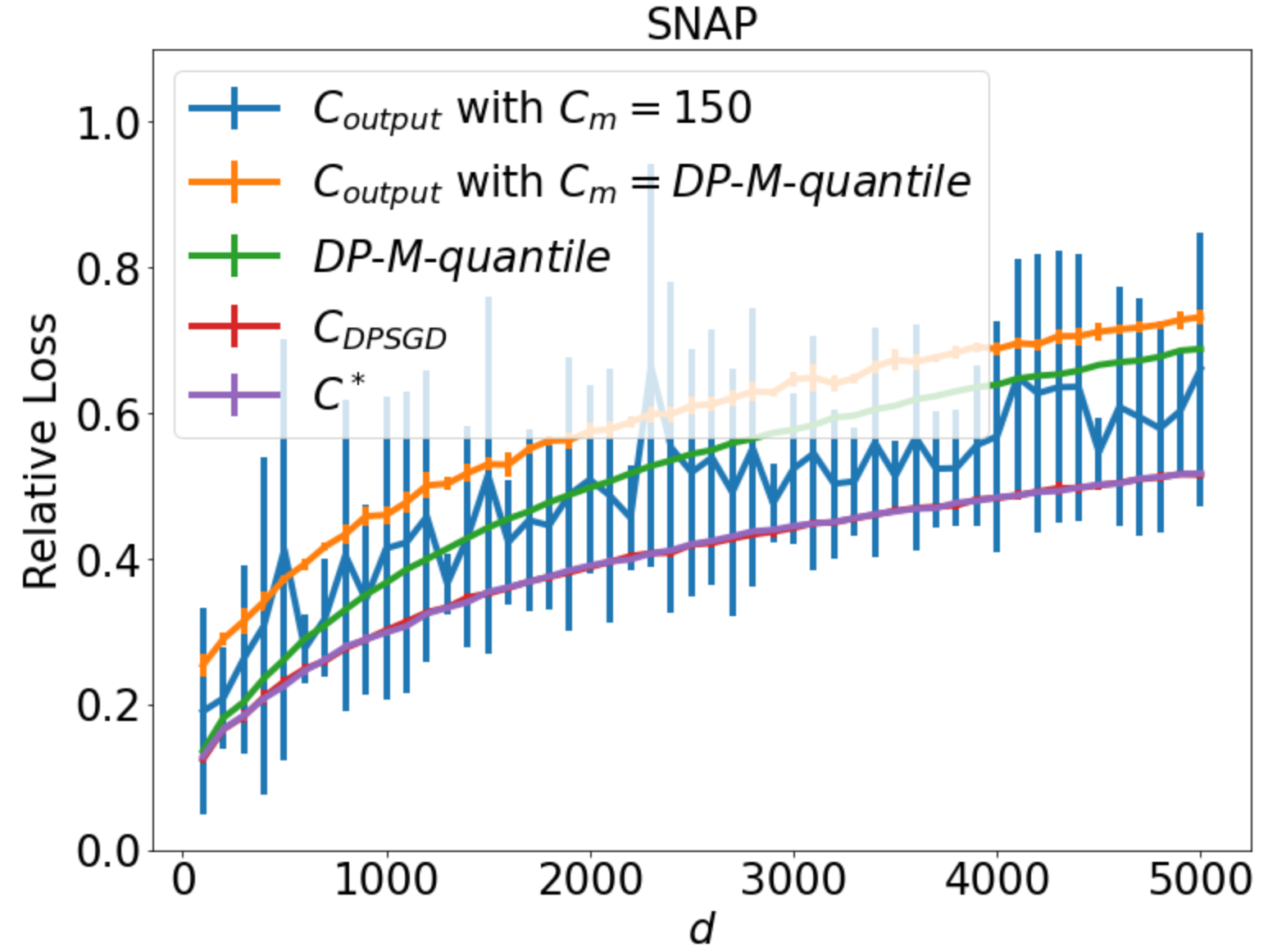}}
  \caption{Histogram estimation over bounded domains. Left: Sentiment140 dataset. Right: SNAP dataset.}
  \label{fig:hist-experiment2}
\end{figure}

\subsection{Setting $s=d$} 
\label{sec:s=d-experiment}

In this section, we demonstrate the performance of the bounded domain algorithm in Section~\ref{sec:estimation_no_assumption} when $s=d$, i.e., bound on sparsity is not known.  We can see that setting $s=d$ still yields a performance close to the true 2-approximation threshold, and much better than the quantile suggested by \cite{amin2019bounding}.
\begin{figure}[h]
    \centering
    \includegraphics[width=0.45\linewidth]{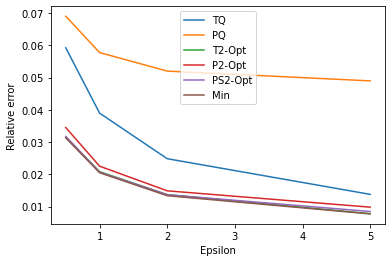}
    \includegraphics[width=0.45\linewidth]{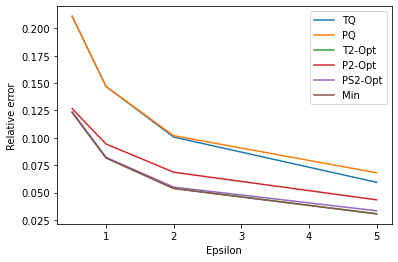}
    \caption{Bounded domain histogram estimation with $d=1000$ (left) and $d=10000$ (right).} \textbf{TQ} is the true (non-private) quantile suggested by \cite{amin2019bounding}. \textbf{PQ} is the private estimate of the quantile. \textbf{P2-Opt} is the result of DP-SGD with $s=d$. \textbf{PS2-Opt} is DP-SGD with $s=0.1d$. \textbf{T2-Opt} is the true 2-approximation threshold. \textbf{Min} is the best threshold (computed by a linear search). 
    \label{fig:k1000}
\end{figure}